\documentclass{article}

\usepackage{microtype}
\usepackage{graphicx}
\usepackage{subfigure}
\usepackage{booktabs} 

\usepackage{hyperref}

\usepackage{amsmath}
\usepackage{amsfonts}
\usepackage{bm}
\usepackage{bbm}
\usepackage{cleveref}
\usepackage{amsthm}
\usepackage{color}
\usepackage{mathtools}
\usepackage{bm}

\newcommand{\bx}{\mathbf{x}}

\newtheorem{thm}{Theorem}
\newtheorem{lem}[thm]{Lemma}

\newtheorem{deff}{Definition}

\newtheorem{exampp}{Example}

\usepackage{xcolor}
\usepackage{framed}
\colorlet{shadecolor}{gray!20}
\newenvironment{examp}
   {\begin{shaded}\begin{exampp}}
   {\end{exampp}\end{shaded}}

\usepackage[accepted]{icml2018}

\icmltitlerunning{Differentiable Compositional Kernel Learning for Gaussian Processes}

\begin{document}

\twocolumn[
\icmltitle{Differentiable Compositional Kernel Learning for Gaussian Processes}




%
%

\begin{icmlauthorlist}
\icmlauthor{Shengyang Sun}{ut,vec}
\icmlauthor{Guodong Zhang}{ut,vec}
\icmlauthor{Chaoqi Wang}{ut,vec}
\icmlauthor{Wenyuan Zeng}{ut,vec,uber}
\icmlauthor{Jiaman Li}{ut,vec}
\icmlauthor{Roger Grosse}{ut,vec}
\end{icmlauthorlist}

\icmlaffiliation{ut}{Department of Computer Science, University of Toronto, Toronto, ON, CA.}
\icmlaffiliation{vec}{Vector Institute.}
\icmlaffiliation{uber}{Uber Advanced Technologies Group, Toronto, ON, CA}

\icmlcorrespondingauthor{Shengyang Sun}{ssy@cs.toronto.edu}

\icmlkeywords{Machine Learning, ICML}

\vskip 0.3in
]



\printAffiliationsAndNotice{}  

\begin{abstract}
The generalization properties of Gaussian processes depend heavily on the choice of kernel, and this choice remains a dark art. We present the Neural Kernel Network (NKN), a flexible family of kernels represented by a neural network. The NKN's architecture is based on the composition rules for kernels, so that each unit of the network corresponds to a valid kernel. It can compactly approximate compositional kernel structures such as those used by the Automatic Statistician \citep{lloyd2014automatic}, but because the architecture is differentiable, it is end-to-end trainable with gradient-based optimization. We show that the NKN is universal for the class of stationary kernels. Empirically we demonstrate NKN's pattern discovery and extrapolation abilities on several tasks that depend crucially on identifying the underlying structure, including time series and texture extrapolation, as well as Bayesian optimization.
\end{abstract}

\section{Introduction}
\label{sec:introduction}

Gaussian processes (GPs) are a powerful and widely used class of models due to their nonparametric nature, explicit representation of posterior uncertainty, and ability to flexibly model a variety of structures in data. 
However, patterns of generalization in GP depend heavily on the choice of kernel function \citep{rasmussen1999evaluation}; different kernels can impose widely varying modeling assumptions, such as smoothness, linearity, or periodicity. Capturing appropriate kernel structures can be crucial for interpretability and extrapolation \citep{duvenaud2013structure,wilson2013gaussian}. Even for experts, choosing GP kernel structures remains a dark art.

\begin{figure}[ht!]
\centering
\hspace{-1em}%
\subfigure[Ground Truth] { 
\includegraphics[height=0.183\textwidth]{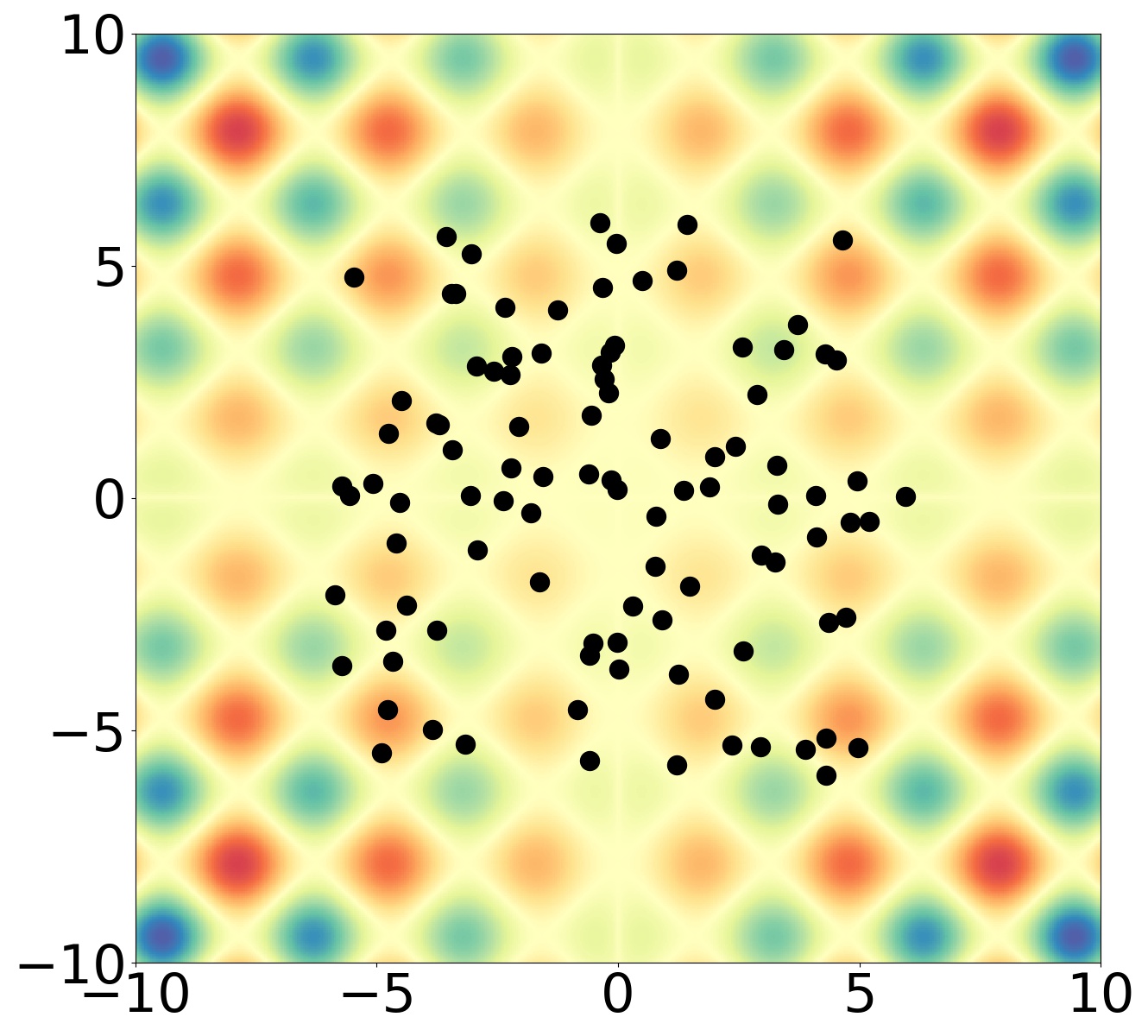} 
} 
\subfigure[NKN Prediction] {
\includegraphics[height=0.18\textwidth]{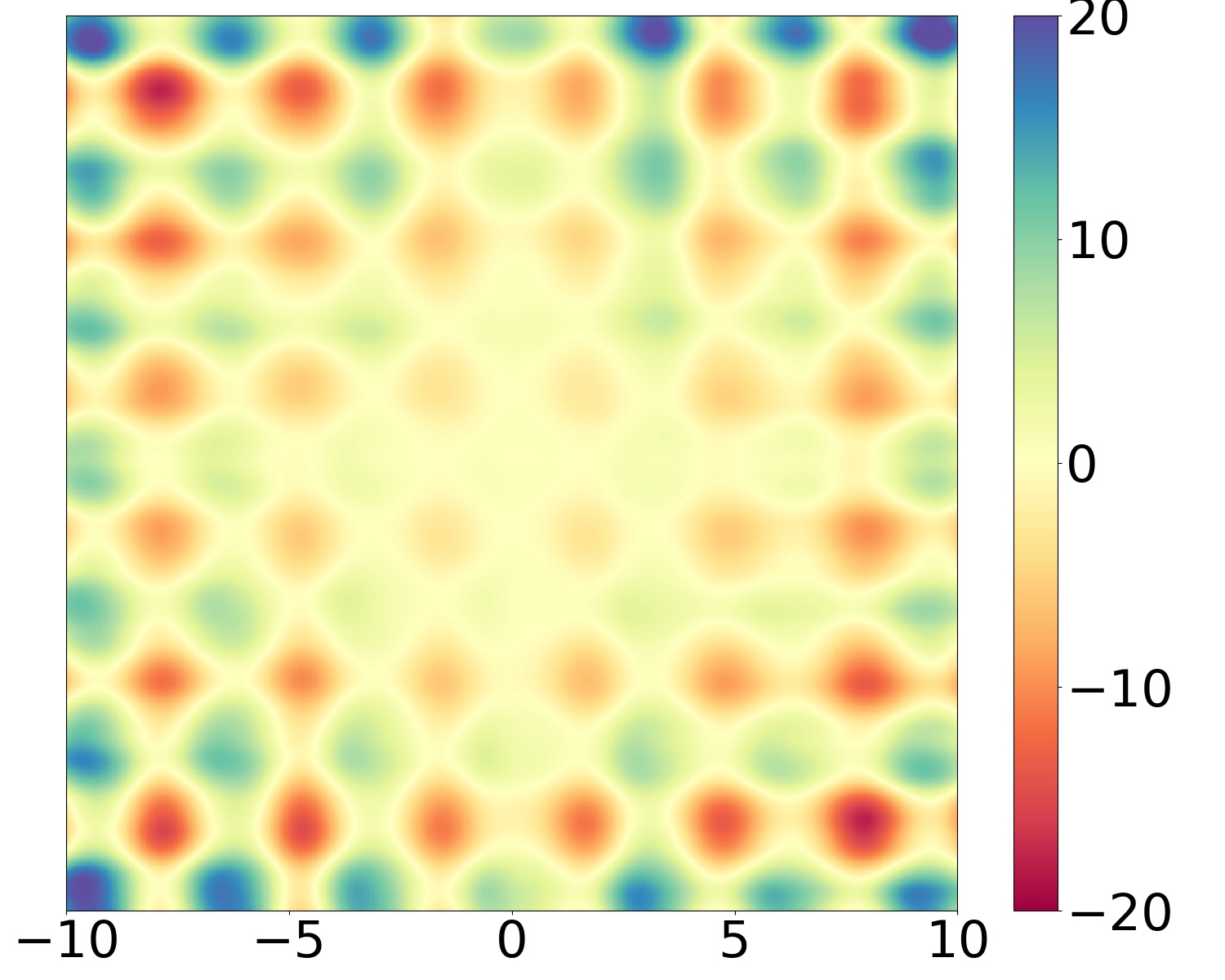} 
}
\vspace{-1em}
\caption{2-D synthetic data (Left) and extrapolation result using our neural kernel network (Right). 
The 2-D function is $y=(\cos(2x_1)+\cos(2x_2)) * \sqrt[]{x_1 x_2}$. Black dots represent 100 training data randomly sampled from $[-6, 6]^2$. This synthetic experiment illustrates NKN's ability to discover and extrapolate patterns.}

\label{fig:intro}
\end{figure}


GPs' strong dependence on kernel structures has motivated work on automatic kernel learning methods. Sometimes this can be done by imposing a specific kind of structure: e.g., \citet{bach2009exploring, duvenaud2011additive} learned kernel structures which were additive over subsets of the variables. A more expressive space of kernels is spectral mixtures \citep{wilson2013gaussian, kom2015generalized, remes2017non}, which are based on spectral domain summations. For example, spectral mixture (SM) kernels \citep{wilson2013gaussian} approximate all stationary kernels using Gaussian mixture models in the spectral domain. Deep kernel learning (DKL) \citep{wilson2016deep} further boosted the expressiveness by transforming the inputs of spectral mixture base kernel with a deep neural network. However, the expressiveness of DKL still depends heavily on the kernel placed on the output layer.

In another line of work, \citet{duvenaud2013structure} defined a context-free grammar of kernel structures based on the composition rules for kernels. Due to its compositionality, this grammar could express combinations of properties such as smoothness, linearity, or periodicity. They performed a greedy search over this grammar to find a kernel struture which matched the input data. Using the learned structures, they were able to produce sensible extrapolations and interpretable decompositions for time series datasets. \citet{lloyd2014automatic} extended this work to an Automatic Statistician which automatically generated natural language reports. All of these results depended crucially on the compositionality of the underlying space. The drawback was that discrete search over the kernel grammar is very expensive, often requiring hours of computation even for short time series.




In this paper, we propose the Neural Kernel Network (NKN), a flexible family of kernels represented by a neural network. The network's first layer units represent primitive kernels, including those used by the Automatic Statistician. Subsequent layers are based on the composition rules for kernels, so that each intermediate unit is itself a valid kernel. The NKN can compactly approximate the kernel structures from the Automatic Statistician grammar, but is fully differentiable, so that the kernel structures can be learned with gradient-based optimization. To illustrate the flexibility of our approach, Figure~\ref{fig:intro} shows the result of fitting an NKN to model a 2-D function; it is able to extrapolate sensibly. 

We analyze the NKN's expressive power for various choices of primitive kernels. We show that the NKN can represent nonnegative polynomial functions of its primitive kernels, and from this demonstrate universality for the class of stationary kernels. Our universality result holds even if the width of the network is limited, analogously to \citet{sutskever2008deep}. Interestingly, we find that the network's representations can be made significantly more compact by allowing its units to represent complex-valued kernels, and taking the real component only at the end.

We empirically analyze the NKN's pattern discovery and extrapolation abilities on several tasks that depend crucially on identifying the underlying structure. The NKN produces sensible extrapolations on both 1-D time series datasets and 2-D textures. It outperforms competing approaches on regression benchmarks. In the context of Bayesian optimization, it is able to optimize black-box functions more efficiently than generic smoothness kernels.

\section{Background}
\label{sec:background}

\subsection{Gaussian Process Regression}
\label{subsec:gp}
A Gaussian process (GP) defines a distribution $p(f)$ over functions $\mathcal{X} \to \mathcal{R}$ for some domain $\mathcal{X}$. For any finite set $\{\mathbf{x}_1, ..., \mathbf{x}_n\} \subset \mathcal{X}$, the function values $\mathbf{f} = (f(\mathbf{x}_1), f(\mathbf{x}_2), ..., f(\mathbf{x}_n))$ have a multivariate Gaussian distribution. Gaussian processes are parameterized by a mean function $\mu(\cdot)$ and a covariance function or kernel function $k(\cdot, \cdot)$. The marginal distribution of function values is given by
\begin{equation}
\mathbf{f} \sim \mathcal{N}(\mu, \mathbf{K}_{XX}),
\end{equation}
where $\mathbf{K}_{XX}$ denotes the matrix of $k(\mathbf{x}_i, \mathbf{x}_j)$ for all $(i, j)$.

Assume we are given a set of training input-output pairs, $\mathcal{D} = \{ (\mathbf{x}_i, y_i) \}_{i=1}^n = (\mathbf{X}, \mathbf{y})$, and each target $y_n$ is generated from the corresponding $f(\mathbf{x}_n)$ by adding independent Gaussian noise; i.e.,
\begin{equation}
y_n = f(\mathbf{x}_n) + \epsilon_n, \quad \epsilon_n \sim \mathcal{N}(0, \sigma^2)
\end{equation}
As the prior on $f$ is a Gaussian process and the likelihood is Gaussian, the posterior on $f$ is also Gaussian. We can use this to make predictions $p(y_*|\mathbf{x}_*, \mathcal{D})$ in closed form:
\begin{equation}
\begin{aligned}
& p(y_*|\mathbf{x}_*, \mathcal{D}) = \mathcal{N}(\mu_*, \sigma_{*}^{2}) \\
& \mu_* = \mathbf{K}_{*X}(\mathbf{K}_{XX} + \sigma^2\mathbf{I})^{-1}\mathbf{y} \\
& \sigma_*^2 = \mathbf{K}_{**} - \mathbf{K}_{*X}(\mathbf{K}_{XX} + \sigma^2\mathbf{I})^{-1}\mathbf{K}_{X*} + \sigma^2
\end{aligned}
\end{equation}
Here we assume zero mean function for $f$. Most GP kernels have several hyperparameters $\theta$ which can be optimized jointly with $\sigma$ to maximize the log marginal likelihood,
\begin{equation}
    \mathcal{L}(\theta) = \ln p(\mathbf{y}|\mathbf{0}, \mathbf{K}_{XX} + \sigma^2 \mathbf{I})
\end{equation}
%

\subsection{Bochner's Theorem}

Gaussian Processes depend on specifying a kernel function $k(x, x')$, which acts as a similarity measure between inputs.
\begin{deff}\label{def:complex-k}
Let $\mathcal{X}$ be a set, and $k$ be a conjugate symmetric function $k: \mathcal{X} \times \mathcal{X} \to \mathbb{C}$ is a positive definite kernel if $\forall \bm{x_1, \cdots, x_n} \in \mathcal{X}$ and $\forall \bm{c_1, \cdots, c_n} \in \mathbb{C}$,
\begin{equation}
    \sum_{i,j=1}^n \bm{c_i \overline{c_j}} k(\bm{x_i, x_j}) \geq 0,
\end{equation}
\end{deff}
where the bar denotes the complex conjugate. Bochner's Theorem \citep{bochner1959lectures} establishes a bijection between complex-valued stationary kernels and positive finite measures using Fourier transform, thus providing an approach to analyze stationary kernels in the spectral domain \citep{wilson2013gaussian, kom2015generalized}.

\begin{thm}\label{thm:bochners}
(Bochner) A complex-valued function $k$ on $\mathbb{R}^d$ is the covariance function of a weakly stationary mean square continuous complex-valued random process on $\mathbb{R}^d$ if and only if it can be represented as 
\begin{equation}
    k(\bm{\tau}) = \int_{\mathbb{R}^P} \exp(2\pi i \bm{w^\top\tau})\psi(\mathrm{d}\bm{w})
\end{equation}
where $\psi$ is a positive and finite measure. If $\psi$ has a density $S(\bm{w})$, then $S$ is called the spectral density or power spectrum of $k$. $S$ and $k$ are Fourier duals.
\end{thm}

\begin{figure*}[t]
\centering
    \includegraphics[width=0.8\textwidth]{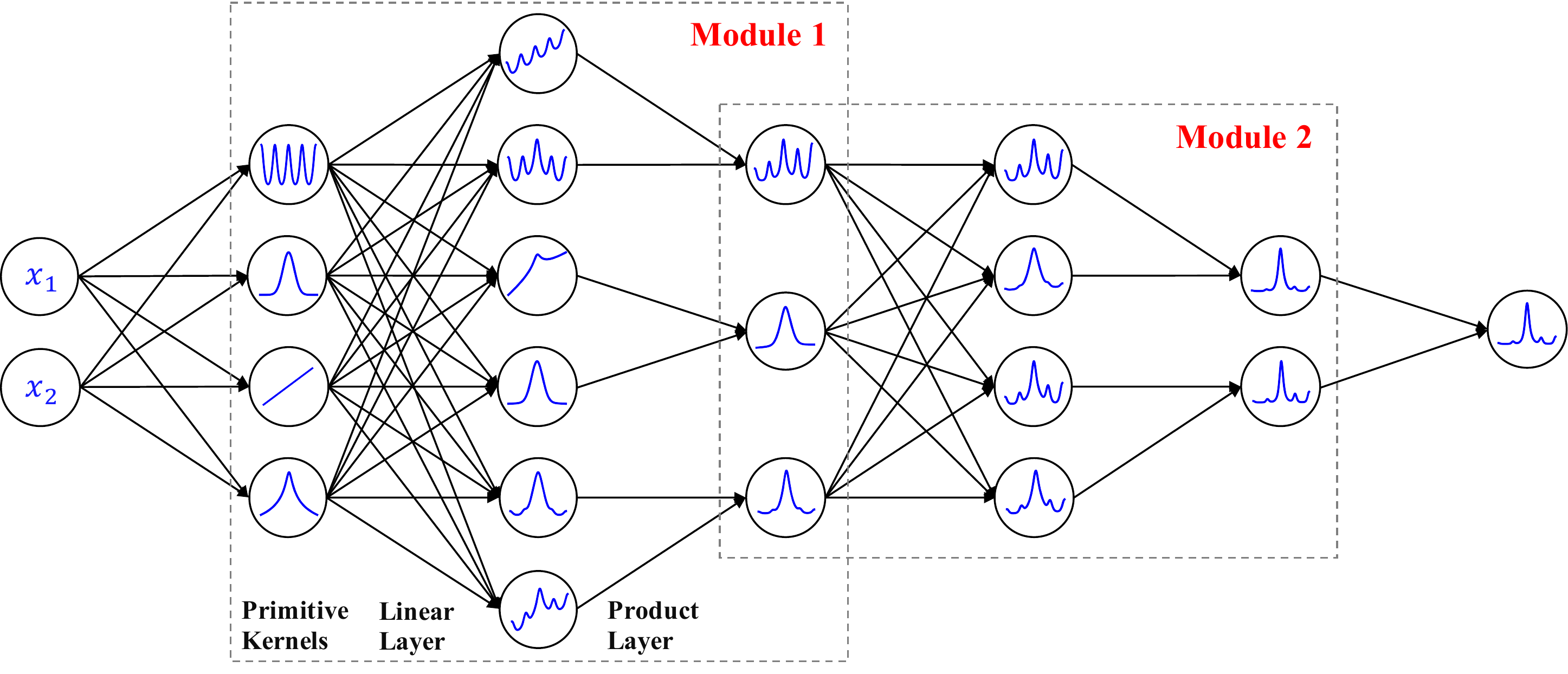}
\caption{Neural Kernel Network: each module consists of a \textbf{Linear} layer and a \textbf{Product} layer. NKN is based on compositional rules for kernels, thus every individual unit itself represents a kernel.}
\label{fig:nkn}
\end{figure*}
\subsection{Automatic Statistician}
\label{sec:autostat}

For compositional kernel learning, the Automatic Statistician \citep{lloyd2014automatic, duvenaud2013structure} used a compositional space of kernels defined as sums and products of a small number of primitive kernels. The primitive kernels included:
\begin{itemize}
\setlength{\itemsep}{-0.1cm}
\item radial basis functions, corresponding to smooth functions. 
$\mathrm{RBF}(\bx, \bx') = \sigma^2 \exp(-\frac{\|\bx-\bx'\|^2}{2l^2})$
\item periodic. $\mathrm{PER}(\bx, \bx') = \sigma^2 \exp(-\frac{2\sin^2(\pi \|\bx-\bx'\| / p)}{l^2})$
\item linear kernel. $\mathrm{LIN}(\bx, \bx') = \sigma^2 \bx^\top\bx'$
\item rational quadratic, corresponding to functions with multiple scale variations. 
$\mathrm{RQ}(\bx, \bx') =  \sigma^2 (1 + \frac{\|\bx-\bx'\|^2}{2\alpha l^2})^{\frac{1}{\alpha}}$
\item white noise. $\mathrm{WN}(\bx, \bx') = \sigma^2 \delta_{\bx, \bx'}$
\item constant kernel. $\mathrm{C}(\bx, \bx') = \sigma^2$
\end{itemize}
The Automatic Statistician searches over the compositional space based on three search operators. 
\begin{enumerate}
\setlength{\itemsep}{-0.203cm}
\item Any subexpression $\mathcal{S}$ can be replaced with $\mathcal{S}+\mathcal{B}$, where $\mathcal{B}$ is any primitive kernel family.
\item Any subexpression $\mathcal{S}$ can be replaced with $\mathcal{S}\times\mathcal{B}$, where $\mathcal{B}$ is any primitive kernel family.
\item Any primitive kernel $\mathcal{B}$ can be replaced with any other primitive kernel family $\mathcal{B}'$.
\end{enumerate}
The search procedure relies on a greedy search: at every stage, it searches over all subexpressions and all possible operators, then chooses the highest scoring combination. To score kernel families, it approximates the marginal likelihood using the Bayesian information criterion \citep{schwarz1978estimating} after optimizing to find the maximum-likelihood kernel parameters.

\section{Neural Kernel Networks}
\label{sec:nkn}

In this section, we introduce the Neural Kernel Network (NKN), a neural net which computes compositional kernel structures and is end-to-end trainable with gradient-based optimization. The input to the network consists of two vectors $\bx_1, \bx_2 \in \mathbb{R}^d$, and the output $k(\bx_1, \bx_2) \in \mathbb{R}$ (or $\mathbb{C}$) is the kernel value. 
Our NKN architecture is based on well-known composition rules for kernels:
\begin{lem}\label{lem:kernel-property} For kernels $k_1, k_2$
\begin{itemize}
\setlength{\itemsep}{-0.1cm}
\item For $\lambda_1, \lambda_2 \in \mathbb{R}^+$, $\lambda_1 k_1 + \lambda_2 k_2$ is a kernel.
\item The product $k_1 k_2$ is a kernel.
\end{itemize}
\end{lem}
We design the architecture such that every unit of the network computes a kernel, although some of those kernels may be complex-valued. 


\subsection{Architecture}
\label{subsec:forward}

The first layer of the NKN consists of a set of primitive kernels. Subsequent layers alternate between linear combinations and products. Since the space of kernels is closed under both operations, each unit in the network represents a kernel. Linear combinations and products can be seen as OR-like and AND-like operations, respectively; this is a common pattern in neural net design \citep{lecun1989backpropagation,poon2011sum}. The full architecture is illustrated in Figure~\ref{fig:nkn}.

{\bf Primitive kernels}. The first layer of the network consists of a set of primitive kernel families with simple functional forms. While any kernels can be used here, we use the RBF, PER, LIN, and RQ kernels from the Automatic Statistician (see Section~\ref{sec:autostat}) because these express important structural motifs for GPs. Each of these kernel families has an associated set of hyperparameters (such as lengthscales or variances), and instantiating the hyperparameters gives a kernel. These hyperparameters are treated as parameters (weights) in this layer of the network, and are optimized with the rest of the network. Note that it may be advantageous to have multiple copies of each primitive kernel so that they can be instantiated with different hyperparameters.



{\bf Linear layers}. The Linear layer closely resembles a fully connected layer in deep neural networks, with each layer $\mathbf{h}_l = \mathbf{W}_l \mathbf{h}_{l-1}$ representing a nonnegative linear combination of units in the previous layer (i.e.~$\mathbf{W}_l$ is a nonnegative matrix). In practice, we use the parameterization $\mathbf{W}_l=\log (1+\exp(\mathbf{A}_l))$ to enforce the nonnegativity constraint. (Here, $\exp$ is applied elementwise.)

The Linear layer can be seen as a OR-like operation: two points are considered similar if either kernel has a high value, while the Linear layer further controls the balance using trainable weights.

{\bf Product layers}. The Product layer introduces multiplication, in that each unit is the product of several units in the previous layer. This layer has a fixed connectivity pattern and no trainable parameters. While this fixed structure may appear restrictive, Section~\ref{sec:universality} shows that it does not restrict the expressiveness of the network.

The Product layer can be seen as an AND-like operation: two points are considered similar if both constituent kernels have large values. 

{\bf Activation functions}. Analogously to ordinary neural nets, each layer may also include a nonlinear activation function, so that $\mathbf{h}_l = f(\mathbf{z}_{l})$, where $\mathbf{z}_l$, the pre-activations, are the result of a linear combination or product. However, $f$ must be selected with care in order to ensure closure of the kernels. Polynomials with positive coefficients, as well as the exponential function $f(z)=e^z$, fulfill this requirement.

{\bf Complex-valued kernels}. Allowing units in NKN to represent complex-valued kernels as in Definition~\ref{def:complex-k} and take the real component only at the end, can make the network's representations significantly more compact. As complex-valued kernels also maintain closure under summation and multiplication \citep{yaglom2012correlation}, additional modifications are unnecessary. In practice, we can include $\exp(i \bm{\mu^\top \tau})$ in our primitive kernels. 

\subsection{Learning}

{\bf Optimization}.
All trainable parameters can be grouped into two categories: (1) parameters of primitive kernels, e.g., lengthscale in an RBF kernel; (2) parameters of Linear layers. We jointly learn these parameters by maximizing the marginal likelihood $\mathcal{L}(\bm{\theta})$. Since the NKN architecture is differentiable, we can jointly fit all parameters using gradient-based optimization.

{\bf Computational Cost}.
NKN introduces small computational overhead. Suppose we have $N$ data points and $m$ connections in the NKN; the computational cost of the forward pass is $O(N^2 m)$. Note that a moderately-sized NKN, as we used in our experiments\footnote{In our experiments, we found 1 or 2 modules work very well. But it might be advantageous to use more modules in other tasks.}, has only tens of parameters, and the main computational bottleneck in training lies in inverting kernel matrix, which is an $O(N^3)$ operation; therefore, NKN incurs only small per-iteration overhead compared to ordinary GP training.

 
\subsection{Universality}
\label{sec:universality}

In this section, we analyze the expressive power of the NKN, and in particular its ability to approximate arbitrary stationary kernels. Our analysis provides insight into certain design decisions for the NKN: in particular, we show that the NKN can approximate some stationary kernels much more compactly if the units of the network are allowed to take complex values. Furthermore, we show that the fixed structure of the product layers does not limit what the network can represent.


\begin{deff}\label{def:pwp}
For kernels $\{k_j\}_{j=1}^n$, a kernel $k$ is positive-weighted polynomial (PWP) of these kernels if $\exists T \in \mathbb{N}$ and $\{w_t, \{p_{tj}\}_{j=1}^n | w_i \in \mathbb{R}^+, p_{tj} \in \mathbb{N}\}_{t=0}^T$, such that
\begin{equation}
    k(x, y) = \sum_{t=1}^T w_t \prod_{j=1}^n k_j^{p_{tj}}
\end{equation}
holds for all $x, y \in \mathbb{R}$. Its degree is $\underset{t}{\max} \sum_{j=1}^n p_{tj}$.
\end{deff}

Composed of summation and multiplication, the NKN naturally forms a positive-weighted polynomial of primitive kernels. Although NKN adopts a fixed multiplication order in the Product layer, the following theorem shows that this fixed architecture doesn't undermine NKN's expressiveness (proof in Appendix~\ref{app:nkn-poly-univ}).

\begin{thm} Given $B$ primitive kernels,
\label{thm:poly-univ}
\begin{itemize}
\item An NKN with width $2B+6$ can represent any PWP of primitive kernels.
\item An NKN with width $2^{Bp+1}$ and $p$ Linear-Product modules can represent any PWP with degree no more than $2^p$.
\end{itemize}
\end{thm}

Interestingly, NKNs can sometimes approximate (real-valued) kernels more compactly if the hidden units are allowed to represent complex-valued kernels, and the real part is taken only at the end. In particular, we give an example of a spectral mixture kernel class which can be represented with an NKN with a single complex-valued primitive kernel, but whose real-valued NKN representation requires a primitive kernel for each mixture component (proof in Appendix~\ref{app:limitation-real-kernel}).

%
\begin{examp} \label{examp: limitation-real}
Define a $d$-dimensional spectral mixture kernel with $n+1$ components, $k^*(\bm{\tau})=\overset{n+1}{\underset{t=1}\sum} \binom{n}{2}^{2t}\cos(4^t \mathbf{1}^\top \bm{\tau})$. Then $\exists \epsilon > 0$, such that $\forall \{\bm{\mu_t} \}_{t=1}^n$, and any PWP of $\{\cos(\bm{\mu_t^\top\tau})\}_{t=1}^n$ denoted as $\bar{k}$,
\begin{equation}
    \underset{\bm{\tau} \in \mathbb{R}^d}\max | \bar{k}(\bm{\tau}) - k^*(\bm{\tau}) | > \epsilon
\end{equation}
In contrast, $k^*$ can be represented as the real part of a PWP of only one complex-valued primitive kernel $e^{i \mathbf{1}^\top \bm{\tau}}$,
\begin{equation}
    k^*(\bm{\tau}) = \Re\{\overset{n+1}{\underset{t=1}{\sum}} \binom{n}{2}^{2t} [e^{i \mathbf{1}^\top \bm{\tau}}]^{4^t} \} 
\end{equation}

\end{examp}

We find that an NKN with small width can approximate any complex-valued stationary kernel, as shown in the following theorem (Proof in Appendix~\ref{app:nkn-stationary-univ}).

\begin{thm}\label{thm:stationary-univ}
For any $d$-dimensional complex-valued stationary kernel $k^*$ and $\epsilon \in \mathbb{R}^+$, $\exists \{\bm{\gamma_j}\}_{j=1}^d, \{\bm{\mu_j}\}_{j=1}^{2d}$, and an NKN $\bar{k}$ with primitive kernels $\{\exp(-2\pi^2 \|\bm{\tau} \odot \bm{\gamma_j} \|^2 )\}_{j=1}^d$, $\{\exp(i\bm{\mu_j}^\top \bm{\tau})\}_{j=1}^{2d}$, and width no more than $6d+6$, such that
\begin{equation}
    \underset{\tau \in \mathbb{R}^d}{\max} |\bar{k}(\bm{\tau})-k^*(\bm{\tau})| < \epsilon
\end{equation}

\end{thm}

Beyond approximating stationary kernels, NKN can also capture non-stationary structure by incorporating non-stationary primitive kernels. In Appendix~\ref{app:nkn-closure}, we prove that with the proper choice of primitive kernels, NKN can approximate a broad family of non-stationary kernels called generalized spectral kernels \citep{kom2015generalized}.

\section{Related Work}
\label{sec:related}

Additive kernels \citep{duvenaud2011additive} are linear combinations of kernels over individual dimensions or groups of dimensions, and are a promising method to combat the curse of dimensionality. While additive kernels need an exponential number of multiplication terms in the input dimension, hierarchical kernel learning (HKL) \citep{bach2009exploring} presents a similar kernel except selecting only a subset to get a polynomial number of terms. However, this subset selection imposes additional optimization difficulty.

Based on Bochner's theorem, there is another a line of work  on designing kernels in the spectral domain, including sparse spectrum kernels (SS) \citep{quia2010sparse}; spectral mixture (SM) kernels \citep{wilson2013gaussian}; generalized spectral kernels (GSK) \citep{kom2015generalized} and generalized spectral mixture (GSM) kernels \cite{remes2017non}. Though these approaches often extrapolate sensibly, capturing complex covariance structure may require a large number of mixture components.

The Automatic Statistician \citep{duvenaud2013structure, lloyd2014automatic, malkomes2016bayesian} used a compositional grammar of kernel structures to analyze datasets and provide natural language reports. In each stage, it considered all production rules and used the one that resulted in the largest log-likelihood improvement. Their model showed good extrapolation for many time series tasks, attributed to the recovery of underlying structure. However, it relied on greedy discrete search over kernel and operator combinations, making it computational expensive, even for small time series datasets.

There have been several attempts \citep{hinton2008using, wilson2016deep} to combine neural networks with Gaussian processes. Specifically, they used a fixed kernel structure on top of the hidden representation of a neural network. This is complementary to our work, which focuses on using neural networks to infer the kernel structure itself. Both approaches could potentially be combined.

Instead of represeting kernel parametrically, \citet{oliva2016bayesian} modeled random feature dimension with stick breaking prior and \citet{tobar2015learning} generated functions as the convolution between a white noise process and a linear filter drawn from GP. These approaches offer much flexibility but also incur challenges in training.


\begin{table*}[t]
\caption{Average test RMSE and log-likelihood for regression benchmarks with random splits.}
\vspace{-1em}
\label{tab:uci-regression}
\begin{center}
\resizebox{\textwidth}{!}{
\begin{sc}
\begin{tabular}{lcccccccc}
\toprule
\textbf{}        & \multicolumn{4}{c}{Test RMSE} & \multicolumn{4}{c}{Test log-likelihood} \\
\textbf{Dataset} & BBB & GP-RBF & GP-SM4 & GP-NKN & BBB & GP-RBF & GP-SM4 & GP-NKN \\
\midrule
Boston  		 & 3.171$\pm$0.149 & 2.753$\pm$0.137 & 2.979$\pm$0.162 & \textbf{2.506$\pm$0.150} & -2.602$\pm$0.031 & -2.434$\pm$0.069 & -2.518$\pm$0.107 & \textbf{-2.394$\pm$0.080} \\
Concrete         & 5.678$\pm$0.087 & 4.685$\pm$0.137 & 3.730$\pm$0.190 & \textbf{3.688$\pm$0.249} & -3.149$\pm$0.018 & -2.948$\pm$0.025 & \textbf{-2.662$\pm$0.053} & -2.842$\pm$0.263 \\
Energy           & 0.565$\pm$0.018 & 0.471$\pm$0.013 & 0.316$\pm$0.018 & \textbf{0.254$\pm$0.020} & -1.500$\pm$0.006 & -0.673$\pm$0.035 & -0.320$\pm$0.089 & \textbf{-0.213$\pm$0.162} \\
Kin8nm           & 0.080$\pm$0.001 & 0.068$\pm$0.001 & \textbf{0.061$\pm$0.000} & 0.067$\pm$0.001 & 1.111$\pm$0.007 & 1.287$\pm$0.007 & \textbf{1.387$\pm$0.006} & 1.291$\pm$0.006 \\
Naval            & \textbf{0.000$\pm$0.000} & \textbf{0.000$\pm$0.000} & \textbf{0.000$\pm$0.000} & \textbf{0.000$\pm$0.000} & 6.143$\pm$0.032 & 9.557$\pm$0.001 & \textbf{9.923$\pm$0.000} & 9.916$\pm$0.000 \\
Pow. Plant       & 4.023$\pm$0.036 & 3.014$\pm$0.068 & 2.781$\pm$0.071 & \textbf{2.675$\pm$0.074} & -2.807$\pm$0.010 & -2.518$\pm$0.020 & -2.450$\pm$0.022 & \textbf{-2.406$\pm$0.023} \\
Wine             & 0.643$\pm$0.012 & 0.597$\pm$0.013 & 0.579$\pm$0.012 & \textbf{0.523$\pm$0.011} & -0.977$\pm$0.017 & 0.723$\pm$0.067 & 0.652$\pm$0.060 & \textbf{0.852$\pm$0.064}\\
Yacht            & 1.174$\pm$0.086 & 0.447$\pm$0.083 & 0.436$\pm$0.070 & \textbf{0.305$\pm$0.060} & -2.408$\pm$0.007 & -0.714$\pm$0.449 & -0.891$\pm$0.523 & \textbf{-0.116$\pm$0.270} \\
\bottomrule
\end{tabular}
\end{sc}
}
\end{center}
\end{table*}
\section{Experiments}
\label{sec:experiments}

We conducted a series of experiments to measure the NKN's predictive ability in several settings: time series, regression benchmarks, and texture images. We focused in particular on extrapolation, since this is a strong test of whether it has uncovered the underlying structure. Furthermore, we tested the NKN on Bayesian Optimization, where model structure and calibrated uncertainty can each enable more efficient exploration. Code is available at \url{git@github.com:ssydasheng/Neural-Kernel-Network.git}



\subsection{Time Series Extrapolation}
We first conducted experiments time series datasets to study extrapolation performance. For all of these experiments, as well as the $2\text{-}d$ experiment in Figure~\ref{fig:intro}, we used the same NKN architecture and training setup (Appendix~\ref{app:toy-details}).

\begin{figure}[h!] 
\centering
\includegraphics[width=0.95\columnwidth]{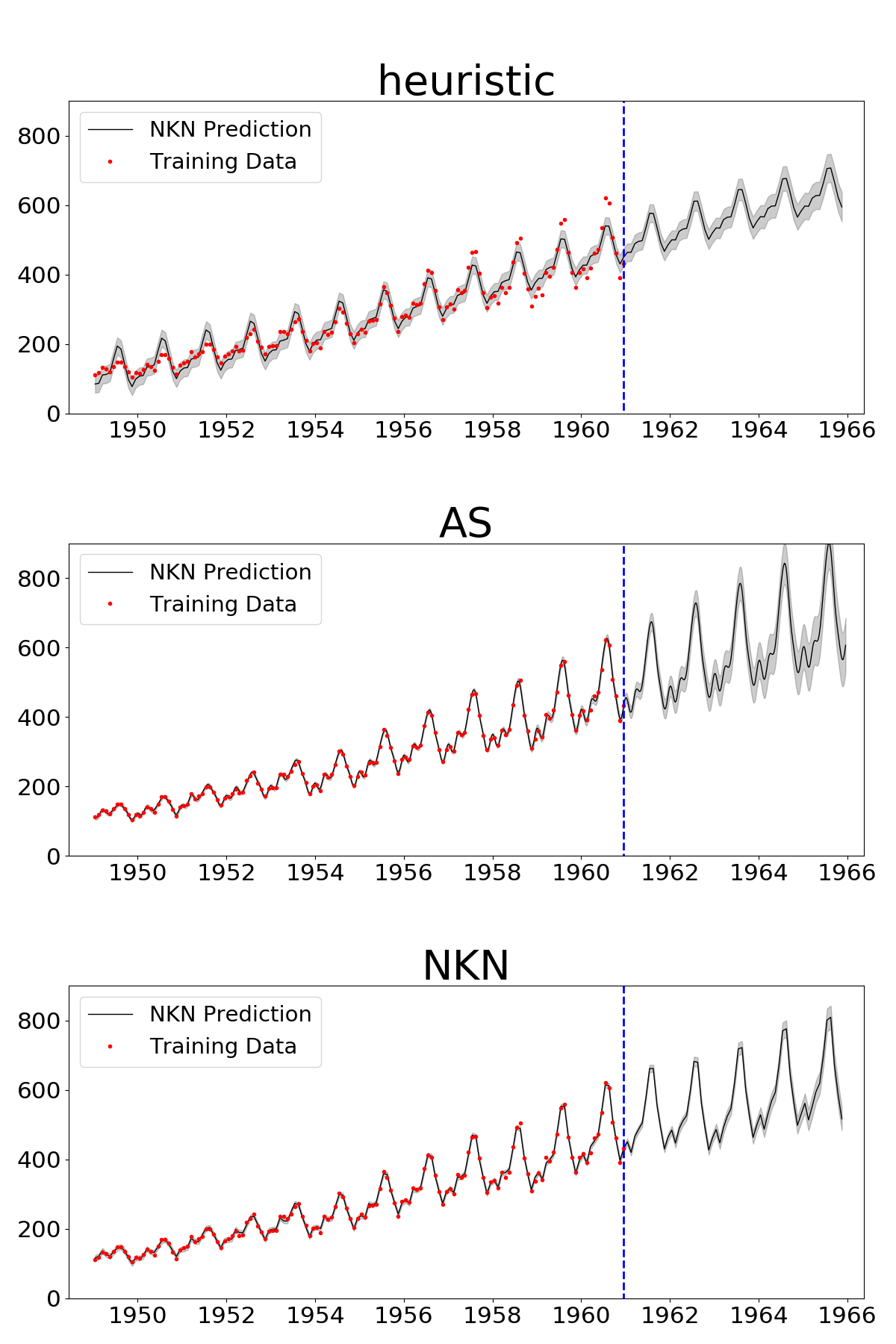}
\caption{Extrapolation results of NKN on the Airline dataset. ``Heuristic'' denotes linear combination of RBF, PER, LIN, and Constant kernels. AS represents Automatic Statistician \citep{duvenaud2013structure}.
The \textcolor{red}{\bf red circles} are the training points, and the curve after the \textcolor{blue}{\bf blue dashed line} is the extrapolation result. Shaded areas represent 1 standard deviation.}
\label{fig:toy_as}
\end{figure}

We validated the NKN on three time series datasets introduced by \citet{duvenaud2013structure}: airline passenger volume (Airline), Mauna Loa atmospheric $\text{CO}_2$ concentration (Mauna), and solar irradiance (Solar). Our focus is on extrapolation, since this is a much better test than interpolation for whether the model has learned the underlying structure.

We compared the NKN with the Automatic Statistician \citep{duvenaud2013structure}; both methods used RBF, RQ, PER and LIN as the primitive kernels. In addition, because many time series datasets appear to contain a combination of seasonal patterns, long-term trends, and medium-scale variability, we also considered a baseline consisting of sums of PER, LIN, RBF, and Constant kernels, with trainable weights and kernel parameters. We refer to this baseline as ``heuristic''.

The results for Airline are shown in Figure~\ref{fig:toy_as}, while the results for Mauna and Solar are shown in Figures~\ref{fig:toy_as_mauna} and \ref{fig:toy_as_solar} in the Appendix. All three models were able to capture the periodic and increasing patterns. However, the heuristic kernel failed to fit the data points well or capture the increasing amplitude, stemming from its lack of PER*LIN structure. In comparison, both AS and NKN fit the training points perfectly, and generated sensible extrapolations. However, the NKN was far faster to train because it avoided discrete search: for the Airline dataset, the NKN took only $201$ seconds, compared with $6147$ seconds for AS.


\label{subsec:toy}

\subsection{Regression Benchmarks}
\label{subsec:reg-cl}

\subsubsection{Random Training/Test Splits}
\label{subsec:reg}

To evaluate the predictive performance of NKN, we first conducted experiments on regression benchmark datasets from the UCI collection \citep{asuncion2007uci}. Following the settings in \citet{hernandez2015probabilistic}, the datasets were randomly split into training and testing sets, comprising 90\% and 10\% of the data respectively. This splitting process was repeated 10 times to reduce variability. We compared NKN to RBF and SM \citep{wilson2013gaussian} kernels, and the popular Bayesian neural network method Bayes-by-Backprop (BBB)~\citep{blundell2015weight}. For the SM kernel, we used 4 mixture components, so we denote it as SM-4. 
For all experiments, the NKN uses $6$ primitive kernels including $2$ RQ, $2$ RBF, and $2$ LIN. The following layers are organized as $\text{Linear}8\text{-}\text{Product}4\text{-}\text{Linear}4\text{-}\text{Product}2\text{-}\text{Linear}1$.\footnote{The number for each layer represents the output dimension.} We trained both the variance and $d$-dimensional lengthscales for all kernels. As a result, for $d$ dimensional inputs, SM-4 has $8d+12$ trainable parameters and NKN has $4d+85$ parameters. 

As shown in Table~\ref{tab:uci-regression}, BBB performed worse than the Gaussian processes methods on all datasets. On the other hand, NKN and SM-4 performed consistently better than the standard RBF kernel in terms of both RMSE and log-likelihoods. Moreover, the NKN outperformed the SM-4 kernel on all datasets other than Naval and Kin8nm.


\begin{figure*}[t]
\centering
\hspace{-1em}%
\subfigure[Stybtang] { \label{fig:bo_sty} 
\includegraphics[height=4.0cm]{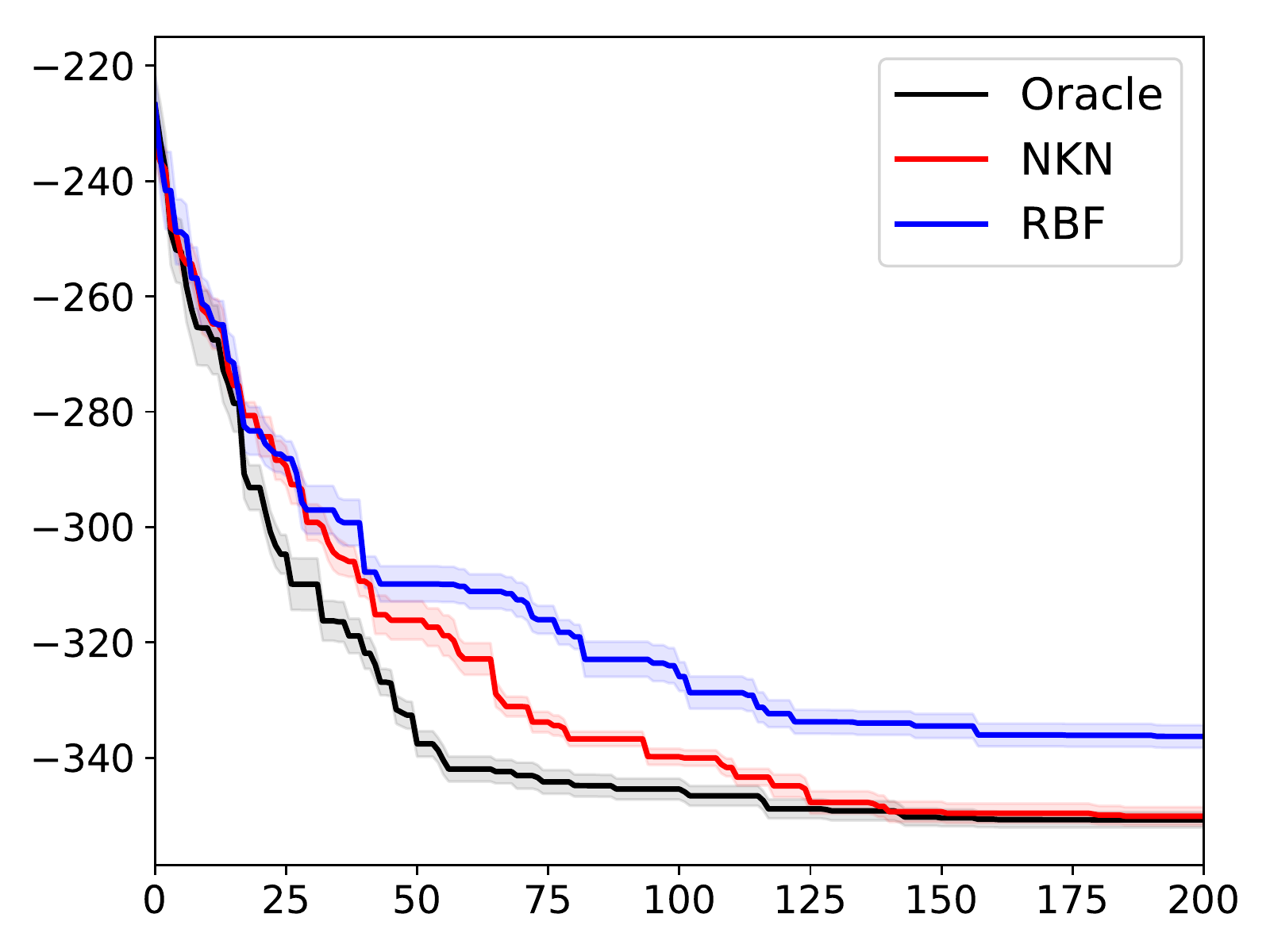} 
} 
\hspace{-1em}
\subfigure[Michalewicz] { \label{fig:bo_mich} 
\includegraphics[height=4.0cm]{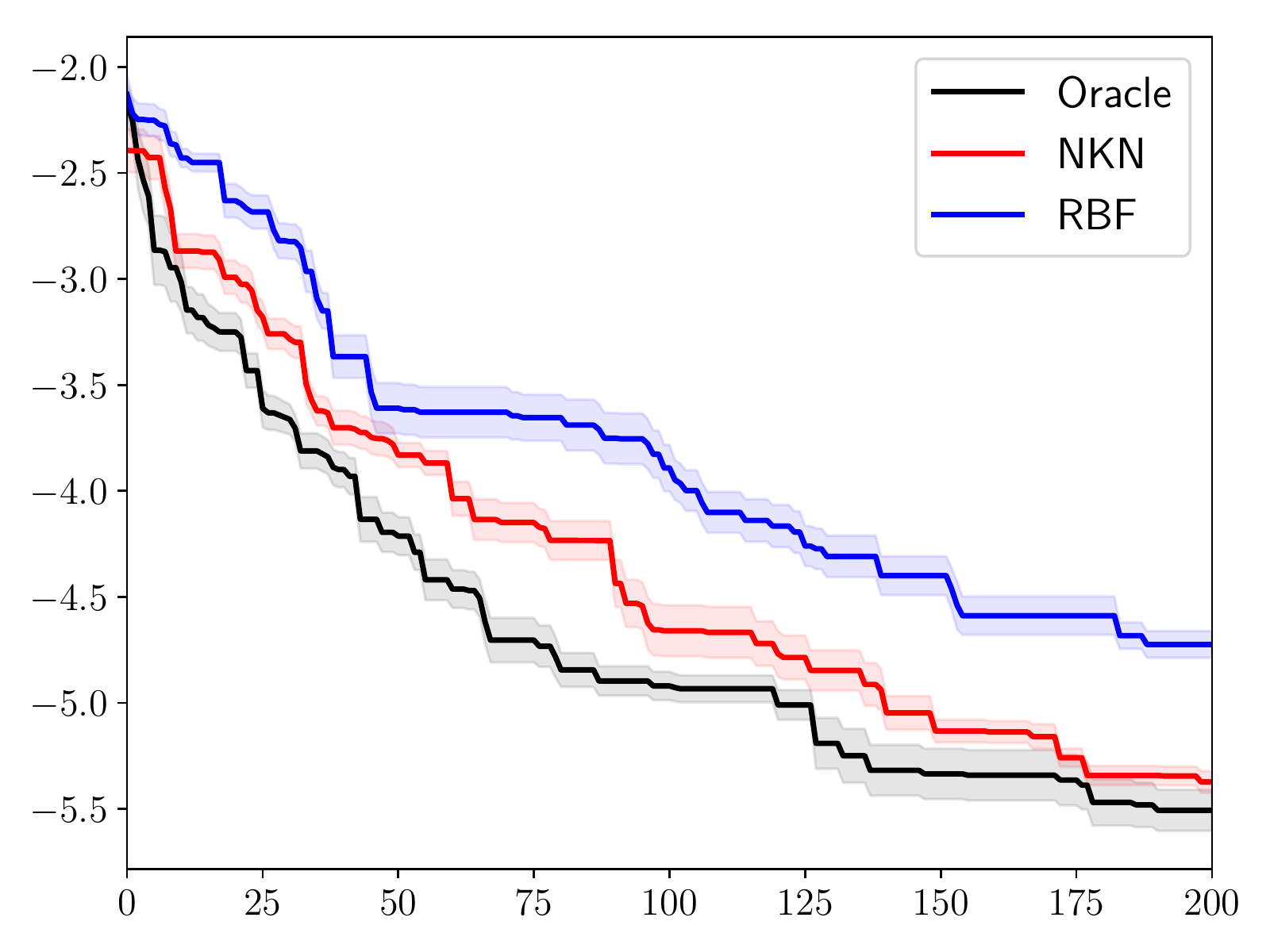} 
}
\hspace{-1em}%
\subfigure[Stybtang-transform] { \label{fig:bo_sty_t} 
\includegraphics[height=4.0cm]{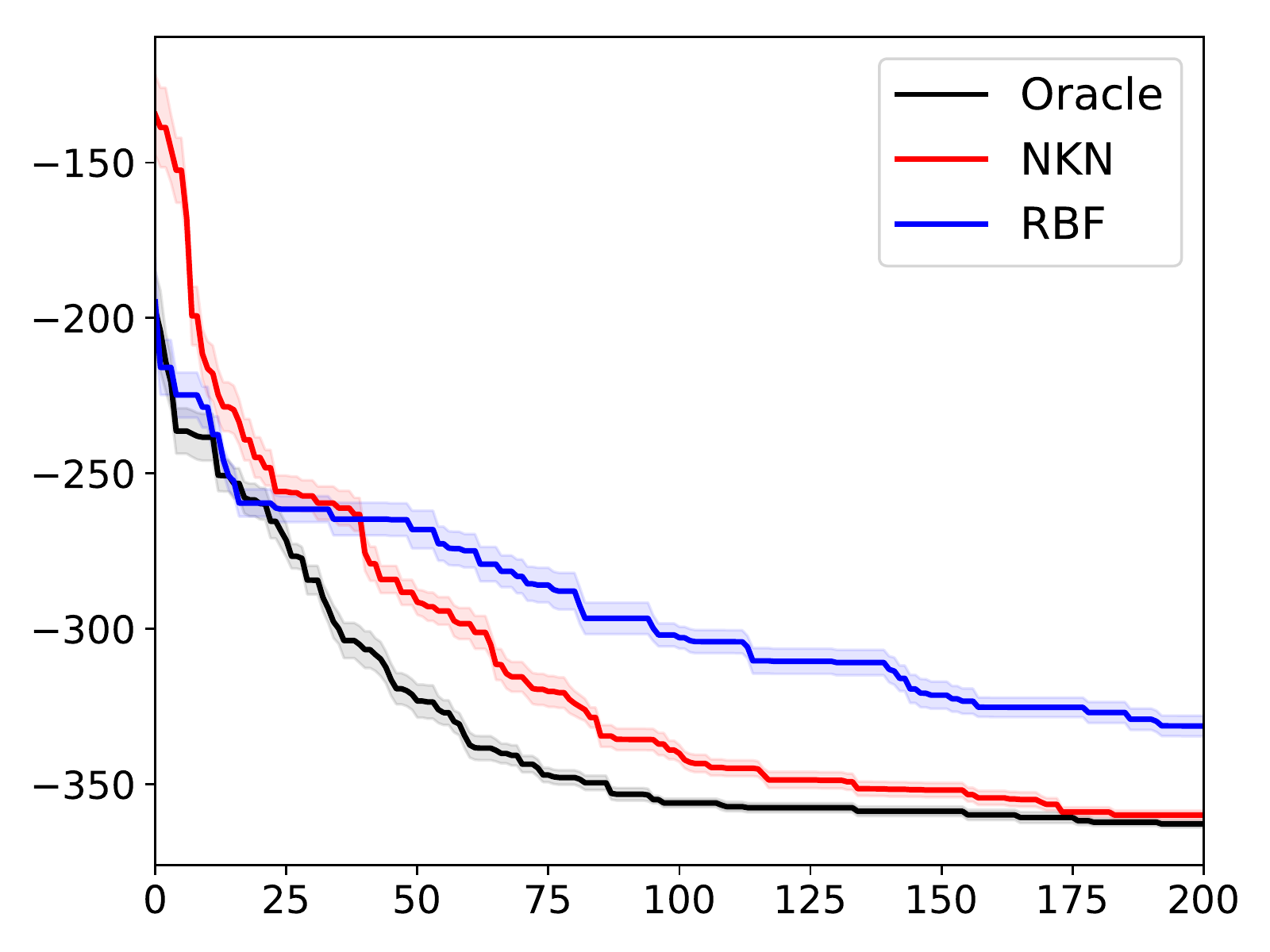} 
} 
\vspace{-1em}
\caption{Bayesian optimization on three tasks. The oracle kernel has the true additive structure of underlying function. Shaded error bars represent 0.2 standard deviations over 10 runs.}
\label{fig:bo} 
\end{figure*}

\subsubsection{Measuring Extrapolation with PCA splits}\label{subsubsec:pca-reg}

Since the experiments just presented used random training/test splits, they can be thought of as measuring interpolation performance. We are also interested in measuring extrapolation performance, since this is a better measure of whether the model has captured the underlying structure. In order to test this, we sorted the data points according to their projection onto the top principal component of the data. The top 1/15 and bottom 1/15 of the data were used as test data, and the remainder was used as training data. 


We compared NKN with standard RBF and SM kernels using the same architectural settings as in the previous section. All models were trained for 20,000 iterations. To select the mixture number of SM kernels, we further subdivided the training set into a training and validation set, using the same PCA-splitting method as described above. For each dataset, we trained SM on the sub-training set using $\{1, 2, 3, 4\}$ mixture components and selected the number based on validation error. (We considered up to 4 mixture components in order to roughly match the number of parameters for NKN.) Then we retrained the SM kernel using the combined training and validation sets. The resulting test RMSE and log-likelihood are shown in Table~\ref{tab:uci-pca-regression}.

As seen in Table~\ref{tab:uci-pca-regression}, all three kernels performed significantly worse compared with Table~\ref{tab:uci-regression}, consistent with the intuition that extrapolation is more difficult that interpolation. However, we can see NKN outperformed SM for most of the datasets. In particular, for small datasets (and hence more chance to overfit), NKN performed better than SM by a substantial margin, with the exception of the \textit{Energy} dataset. This demonstrates the NKN was better able to capture the underlying structure, rather than overfitting the training points.
%

\begin{table}[t]
\caption{Test RMSE and log-likelihood for the PCA-split regression benchmarks. N denotes the number of data points.}
\vspace{-1em}
\label{tab:uci-pca-regression}
\begin{center}
\resizebox{0.5\textwidth}{!}{
\begin{sc}
\begin{tabular}{lccccccc}
\toprule
\textbf{}        & \textbf{} & \multicolumn{3}{c}{Test RMSE} & \multicolumn{3}{c}{Test log-likelihood} \\
\textbf{Dataset} & N         & GP-RBF & GP-SM & GP-NKN       & GP-RBF  & GP-SM   & GP-NKN \\
\midrule
Boston  		 & 506        & 6.390 & 8.600  & \textbf{4.658} & -4.063 & -4.404   & \textbf{-3.557}  \\
Concrete         & 1031        & 8.531 & 7.591  & \textbf{6.242} & -3.246 & -3.285   & \textbf{-3.112} \\
Energy           & 768        & 0.974 & \textbf{0.447} & 0.459  & -1.297 & \textbf{-0.564} & -0.649   \\
Kin8nm           & 8192        & 0.093 & \textbf{0.065} & 0.086  & 0.998  & \textbf{1.322} &  1.057     \\
Naval            & 11934        & 0.000 & \textbf{0.000} & 0.000  & 7.222  & \textbf{9.037} &  6.442    \\
Pow. Plant       & 9568        & 4.768 & 3.931  & \textbf{3.742} & -3.076 & -2.877 &  \textbf{-2.763}   \\
Wine             & 1599        & 0.701 & 0.660  & \textbf{0.650} & -1.076 & -1.002   & \textbf{-0.972} \\
Yacht            & 308        & 1.190 & 1.736  & \textbf{0.528} & -2.896 & -2.768   & \textbf{-0.694} \\
\bottomrule
\end{tabular}
\end{sc}
}
\end{center}
\end{table}


\subsection{Bayesian Optimization}
\label{subsec:bo}

\begin{figure*}[t]
\centering
\subfigure[Observations] { \label{fig:plate-training} 
\includegraphics[height=2.4cm]{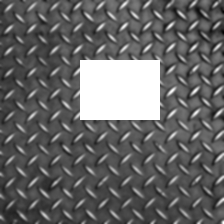} 
} 
\subfigure[Ground Truth] { \label{fig:plate-full} 
\includegraphics[height=2.4cm]{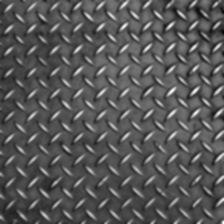} 
}
\subfigure[NKN-4] { \label{fig:plate-nkn} 
\includegraphics[height=2.4cm]{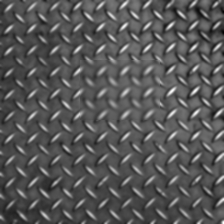} 
} 
\subfigure[RBF] { \label{fig:plate-se} 
\includegraphics[height=2.4cm]{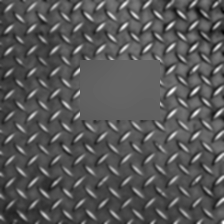} 
}
\subfigure[PER] { \label{fig:plate-per} 
\includegraphics[height=2.4cm]{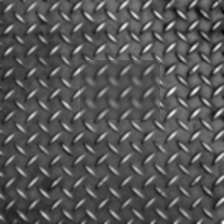} 
}
\subfigure[SM-10] { \label{fig:plate-sm} 
\includegraphics[height=2.4cm]{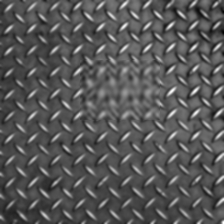} 
} 
\subfigure[Observations] { \label{fig:pave-training} 
\includegraphics[height=2.4cm]{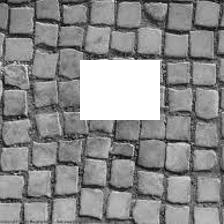} 
} 
\subfigure[Ground Truth] { \label{fig:pave-full} 
\includegraphics[height=2.4cm]{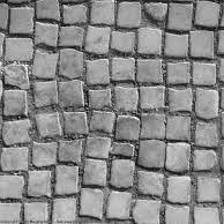} 
}
\subfigure[NKN-4] { \label{fig:pave-nkn} 
\includegraphics[height=2.4cm]{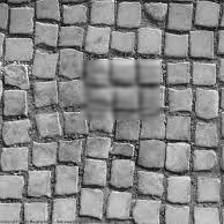} 
} 
\subfigure[RBF] { \label{fig:pave-se} 
\includegraphics[height=2.4cm]{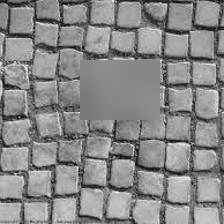} 
}
\subfigure[PER] { \label{fig:pave-per} 
\includegraphics[height=2.4cm]{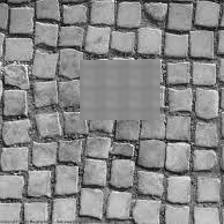} 
}
\subfigure[SM-10] { \label{fig:pave-sm} 
\includegraphics[height=2.4cm]{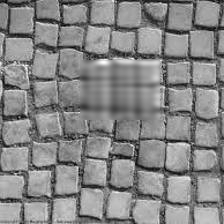} 
} 
\vspace{-1em}
\caption{Texture Extrapolation on metal thread plate (top) and paved pattern (bottom).}
%
\label{fig:rl} 
\end{figure*}

Bayesian optimization \citep{brochu2010tutorial, snoek2012practical} is a technique for optimization expensive black-box functions which repeatedly queries the function, fits a surrogate function to past queries, and maximizes an acquisition function to choose the next query. It's important to model both the predictive mean (to query points that are likely to perform well) and the predictive variance (to query points that have high uncertainty). Typically, the surrogate functions are estimated using a GP with a simple kernel, such as Matern. But simple kernels lead to inefficient exploration due to the curse of dimensionality, leading various researchers to consider additive kernels \citep{kandasamy2015high, gardner2017discovering, wang2017batched}. Since additivity is among the patterns the NKN can learn, we were interested in testing its performance on Bayesian optimization tasks with additive structure. We used Expectated Improvement (EI) to perform BO.

Following the protocol in \citet{kandasamy2015high, gardner2017discovering, wang2017batched}, we evaluated the performance on three toy function benchmarks with additive structure, 
\begin{equation}
    f(\bx) = \sum_{i=1}^{|P|} f_i(\bx[P_i])
\end{equation}

The $d$-dimensional Styblinski-Tang function and Michalewicz function have fully additive structure with independent dimensions. In our experiment, we set $d=10$ and explored the function over domain $[-4, 4]^d$ for Styblinski-Tang and $[0, \pi]^d$ for Michalewicz. We also experimented with a transformed Styblinski-Tang function, which applies Styblinski-Tang function on partitioned dimension groups. 

For modelling additive functions with GP, the kernel can decompose as a summation between additive groups as well. $k(\bx, \bx')=\sum_{i=1}^{|P|} k_i(\bx[P_i], \bx'[P_i])$. 
We considered an oracle kernel, which was a linear combination of RBF kernels corresponding to the true additive structure of the function. Both the kernel parameters and the combination coefficients were trained with maximum likelihood. 
We also tested the standard RBF kernel without additive structure. For the NKN, we used $d$ RBF kernels over individual input dimensions as the primitive kernels. The following layers were arranged as $\text{Linear}8\text{-}\text{Product}4\text{-}\text{Linear}4\text{-}\text{Product}2\text{-}\text{Linear}1$. Note that, although the primitive kernels corresponded to seperate dimensions, NKN can represent additive structure through these linear combination and product operations. In all cases, we used Expected Improvement as the acquisition function.

As shown in Figure~\ref{fig:bo}, for all three benchmarks, the oracle kernel not only converged faster than RBF kernel, but also found smaller function values by a large margin. In comparsion, we can see that although NKN converged slower than oracle in the beginning, it caught up with oracle eventually and reached the same function value. This suggests that the NKN is able to exploit additivity for Bayesian optimization.

\subsection{Texture Extrapolation}
\label{subsec:img-paint}

Based on \citet{wilson2014fast}, we evaluated the NKN on texture exploration, a test of the network's ability to learn local correlations as well as complex quasi-periodic patterns. From the original $224 \times 224$ images, we removed a $60 \times 80$ region, as shown in Figure~\ref{fig:plate-training}. From a regression perspective, this corresponds to 45376 training examples and 4800 test examples, where the inputs are 2-D pixel locations and the outputs are pixel intensities. To scale our algorithms to this setting, we used the approach of \citet{wilson2014fast}. In particular, we exploited the grid structure of texture images to represent the kernel matrix for the full image as a Kronecker product of kernel matrices along each dimension \citep{saatcci2012scalable}. Since some of the grid points are unobserved, we followed the algorithm in \citet{wilson2014fast} complete the grid with imaginary observations, and placed infinite measurement noise on these observations. 

To reconstruct the missing region, we used an NKN with 4 primitive kernels: LIN, RBF, RQ, and PER. As shown in Figure~\ref{fig:plate-nkn}, our NKN was able to learn and extrapolate complex image patterns. As baselines, we tested RBF and PER kernels; those results are shown in Figure~\ref{fig:plate-se} and Figure~\ref{fig:plate-per}. The RBF kernel was unable to extrapolate to the missing region, while the PER kernel was able to extrapolate beyond the training data since the image pattern is almost exactly periodic.
We also tested the spectral mixture (SM) kernel, which has previously shown promising results in texture extrapolation. Even with 10 mixture components, its extrapolations were blurrier compared to those of the NKN.
The second row shows extrapolations on an irregular paved pattern, which we believe is more difficult. The NKN still provided convincing extrapolation. By contrast, RBF and PER kernels were unable to capture enough information to reconstruct the missing region. 



\section{Conclusion}
\label{sec:conclusion}

We proposed the Neural Kernel Network (NKN), a differentiable architecture for compositional kernel learning. Since the architecture is based on the composition rules for kernels, the NKN can compactly approximate the kernel structures from the Automatic Statistician (AS) grammar. But because the architecture is differentiable, the kernel can be learned orders-of-magnitude faster than the AS using gradient-based optimization. We demonstrated the universality of the NKN for the class of stationary kernels, and showed that the network's representations can be made significantly more compact using complex-valued kernels. Empirically, we found the NKN is capable of pattern discovery and extrapolation in both 1-D time series datasets and 2-D textures, and can find and exploit additive structure for Bayesian Optimization.

\section*{Acknowledgements}
We thank David Duvenaud and Jeongseop Kim for their insightful comments and discussions on this project. SS was supported by a Connaught New Researcher Award and a Connaught Fellowship. GZ was supported by an NSERC Discovery Grant.

\bibliography{ref}
\bibliographystyle{icml2018}

\clearpage

\appendix

\section{Complex-Valued Non-Stationary Kernels}
\label{app:nonstationary-kernel}

Beyond stationary kernels, Generalized Bochner's Theorem \citep{yaglom2012correlation, kakihara1985note, genton2001classes} presents a spectral representation for a larger class of kernels. As stated in \citet{yaglom2012correlation}, Generalized Bochner's Theorem applies for most bounded kernels except for some special counter-examples. 

\begin{thm}\label{thm:nonstationary-bochners}
(Generalized Bochner) A complex-valued bounded continuous function $k$ on $\mathbb{R}^d$ is the covariance function of a mean square continuous complex-valued random process on $\mathbb{R}^d$ if and only if it can be represented as 
\begin{equation}
    k(\bm{x, y}) = \int_{\mathbb{R}^d\times \mathbb{R}^d} e^{2\pi i(\bm{w_1}^\top\bm{x}-\bm{w_2}^\top\bm{y})}\psi(\mathrm{d}\bm{w_1}, \mathrm{d}\bm{w_2})
\end{equation}
where $\psi$ is a Lebesgue-Stieltjes measure associated with some positive semi-definite bounded symmetric function $S(\bm{w_1, w_2})$.
\end{thm}
We call kernels with this property \textit{complex-valued non-stationary} (CvNs) kernels.

Actually, when the spectral measure $\psi$ has mass concentrated along the diagonal $\bm{w_1}=\bm{w_2}$, Bochner's theorem is recovered. Similarly, for such a CvNs kernel, the closure under summation and multiplication still holds (Proof in Appendix~\ref{app:nons-complex-closure}):

\begin{lem}\label{lem:nonstationary-closure} For CvNs kernels $k, k_1, k_2$,
\begin{itemize}
\item For $\lambda_1, \lambda_2 \in \mathbb{R}^*$, $\lambda_1 k_1 + \lambda_2 k_2$ is a CvNs kernel.
\item Product $k_1 k_2$ is a CvNs kernel.
\item Real part $\Re\{k\}$ is a real-valued kernel.
\end{itemize}
\end{lem}

Finally, the next theorem justifies the use of complex-valued kernels and real-valued kernels together in an NKN.
\begin{thm}\label{thm:nkn-closure}
An NKN with primitive kernels which are either real-valued kernels or CvNs kernels has all the nodes' real parts be valid real-valued kernels.
\end{thm}
\begin{proof}
Assume primitive kernels have real-valued kernels $r_1, \cdots, r_t$ and CvNs kernels $c_1, \cdots, c_s$. Any node in an NKN can be written as $\mathrm{Poly}_{+}(r_1, \cdots, r_t, c_1, \cdots, c_s)$.\footnote{$\mathrm{Poly}_+$ represents a positive-weighted polynomial of primitive kernels.} It is sufficient to show that for every purely multiplication term in $\mathrm{Poly}_{+}(r_1, \cdots, r_t, c_1, \cdots, c_s)$, its real part is a valid kernel. As the real part equals $\prod_{i=1}^t r_i^{n_i} \times \Re(\prod_{j=1}^s c_j^{m_j})$ and Lemma~\ref{lem:stationary-closure} shows $\Re(\prod_{j=1}^s c_j^{m_j})$ is a valid real-valued kernel,  therefore the whole real part is also a valid kernel.
\end{proof}


\section{Complex-Valued Stationary-Kernel Closure}
\label{app:complex-closure}
\begin{lem}\label{lem:stationary-closure} For complex-valued stationary kernels $k, k_1, k_2$,
\begin{itemize}
\item For $\lambda_1, \lambda_2 \in \mathbb{R}^*$, $\lambda_1 k_1 + \lambda_2 k_2$ is a complex-valued stationary kernel.
\item Product $k_1 k_2$ is a complex-valued stationary kernel.
\item Real part $\Re\{k\}$ is a real-valued kernel.
\end{itemize}
\end{lem}
\begin{proof}
A complex-valued stationary kernel $k$ is equivalent to being the covariance function of a weakly stationary mean square continuous complex-valued random process on $\mathbb{R}^P$, which also means that for any $N\in \mathbb{Z}^+, \{\bm{x_i}\}_{i=1}^N$, we have $\bm{K} = \{k(\bm{x_i} - \bm{x_j})\}_{0\leq i,j\leq N}$ is a complex-valued positive semi-definite matrix.

Thus for complex-valued stationary kernels $k_1$ and $k_2$, and any $N\in \mathbb{Z}^+, \{\bm{x_i}\}_{i=1}^N$, we have complex-valued PSD matrices $\bm{K}_1$ and $\bm{K}_2$. Then product $k=k_1 k_2$ will have matrix $\bm{K} = \bm{K}_1 \circ \bm{K}_2$. For PSD matrices, assume eigenvalue decompositions
\begin{equation}
\begin{aligned}
    \bm{K}_1 = \sum_{i=1}^N \lambda_i \bm{u}_i \bm{u}_i^* \;\;\;
    \bm{K}_2 = \sum_{i=1}^N \gamma_i \bm{v}_i \bm{v}_i^* 
\end{aligned}
\end{equation}
where eigenvalues $\lambda_i$ and $\gamma_i$ are non-negative real numbers based on positive semi-definiteness. Then
\begin{equation}
\begin{aligned}
    \bm{K} &= \bm{K}_1 \circ \bm{K}_2 = \sum_{i=1}^N \lambda_i \bm{u}_i \bm{u}_i^* \circ \sum_{i=1}^N \gamma_i \bm{v}_i \bm{v}_i^* \\
    &= \sum_{i,j=1}^N \lambda_i \gamma_j (\bm{u}_i \circ \bm{v}_j)(\bm{u}_i \circ \bm{v}_j)^*
\end{aligned}
\end{equation}
$\bm{K}$ is also a complex-valued PSD matrix, thus $\bm{K}$ is also a complex-valued stationary kernel.

For sum $k=k_1 + k_2$, $\bm{K} = \bm{K}_1 + \bm{K}_1$ is also a complex-valued PSD.

Hence we proved the closure of complex-valued kernels under summation and multiplication. For a complex-valued kernel $k$ and arbitrarily $N$ complex-valued numbers, as shown above 
\begin{equation}
\begin{aligned}
    \bm{K} &= \sum_{k=1}^N \lambda_i \bm{u}_k \bm{u}_k^* = \sum_{k=1}^N \lambda_i (\bm{a}_k+ i \bm{b}_k) (\bm{a}_k - i \bm{b}_k)^\top \\
    &= \sum_{k=1}^N \lambda_i (\bm{a}_k \bm{a}_k^\top + \bm{b}_k \bm{b}_k^\top +i (\bm{b}_k \bm{a}_k^\top - \bm{a}_k \bm{b}_k^\top)) \\
    &= \sum_{k=1}^N \lambda_i (\bm{a}_k \bm{a}_k^\top + \bm{b}_k\bm{b}_k^\top) + i \sum_{k=1}^N \lambda_i (\bm{b}_k \bm{a}_k^\top - \bm{a}_k \bm{b}_k^\top).
\end{aligned}
\end{equation}
Here $\bm{a}_k, \bm{b}_k$ are all real vectors. Therefore, $\Re(\bm{K})$ is a real PSD matrix, thus $\Re\{k\}$ is a real-valued stationary kernel.

Note: Kernel's summation, multiplication and real part correspond to summation, convolution and even part in the spectral domain. Because the results of applying these spectral-domain operations to a positive function are still positive, and the only requirement for spectral representation is being positive as in Theorem~\ref{thm:bochners}, the closure is natural.
\end{proof}

\section{Proof for Complex-Valued Kernel Closure}
\label{app:nons-complex-closure}
This section provides proof for Lemma~\ref{lem:nonstationary-closure}.

\begin{proof}
First, for $k(\bm{x}, \bm{y})$ to be a valid kernel ($k(\bm{x}, \bm{y})=k(\bm{y}, \bm{x})^*$), $S(\bm{w_1}, \bm{w_2})$ has to be symmetric, which means $S(\bm{w_1}, \bm{w_2}) = S(\bm{w_2}, \bm{w_1})$.\\
\begin{equation}
\begin{aligned}
    k(\bm{y}, \bm{x})^* &= [\int e^{2\pi i (\bm{w_1}^\top\bm{y}-\bm{w_2}^\top\bm{x})} S(\bm{w_1}, \bm{w_2}) \mathrm{d}\bm{w_1}\mathrm{d}\bm{w_2}]* \\
    &= \int e^{2\pi i (\bm{w_2}^\top\bm{x}-\bm{w_1}^\top\bm{y})} S(\bm{w_1}, \bm{w_2}) \mathrm{d}\bm{w_1}\mathrm{d}\bm{w_2} \\
    &= \int e^{2\pi i (\bm{w_1}^\top\bm{x}-\bm{w_2}^\top\bm{y})} S(\bm{w_2}, \bm{w_1}) \mathrm{d}\bm{w_1}\mathrm{d}\bm{w_2}
\end{aligned}
\end{equation}
Compared to
\begin{equation}
    k(\bm{x}, \bm{y}) = \int e^{2\pi i (\bm{w_1}^\top\bm{x}-\bm{w_2}^\top\bm{y})} S(\bm{w_1}, \bm{w_2}) \mathrm{d}\bm{w_1}\mathrm{d}\bm{w_2}
\end{equation}
We have the requirement that $S(\bm{w_1}, \bm{w_2}) = S(\bm{w_2}, \bm{w_1})$.

Next we will show the correspondence between kernel operations and spectral operations. Denote $\bm{z}=\begin{bmatrix}x\\-y\end{bmatrix}$.
\begin{equation}
\begin{aligned}
    &k_1(\bm{x}, \bm{y}) + k_2(\bm{x}, \bm{y}) \\
    &= \int e^{2\pi i (\bm{w_1}^\top\bm{x}-\bm{w_2}^\top\bm{y})}(S_1(\bm{w_1}, \bm{w_2}) + S_2(\bm{w_1}, \bm{w_2}) )\mathrm{d}\bm{w_1}\mathrm{d}\bm{w_2} \\
    &k_1(\bm{x}, \bm{y})  k_2(\bm{x}, \bm{y}) \\
    &= \int_{\bm{w^1}} \int_{\bm{w^2}} e^{2\pi i \bm{z}^\top\bm{w}^1}S_1(\bm{w}^1) e^{2\pi i \bm{z}^\top\bm{w}^2}S_1(\bm{w}^2) \mathrm{d}\bm{w}^1\mathrm{d}\bm{w}^2 \\
    &=\int_{\bm{w}^1}\int_{\bm{w}_2} e^{2\pi i \bm{z}^\top(\bm{w}^1+\bm{w}^2)}S_1(\bm{w}^1)S_2(\bm{w}^2)\mathrm{d}\bm{w}^1\mathrm{d}\bm{w}^2 \\
    &= \int_{\bm{w}^1}\int_{\bm{w}} e^{2\pi i \bm{z}^\top\bm{w}}S_1(\bm{w}^1)S_2(\bm{w}-\bm{w}^1)\mathrm{d}\bm{w}^1\mathrm{d}\bm{w} \\
    &= \int_w e^{2\pi i \bm{z}^\top\bm{w}} \int_{\bm{w}^1} S_1(\bm{w}^1)S_2(\bm{w}-\bm{w}^1)\mathrm{d}\bm{w}^1\mathrm{d}\bm{w} \\
    &= \int_w e^{2\pi i \bm{z}^\top\bm{w}}  S_1 * S_2 \mathrm{d}\bm{w} 
\end{aligned}
\end{equation}
Where $\bm{w^1}=\begin{bmatrix}w_1^1\\ w_2^1\end{bmatrix}$, and $\bm{w^2}=\begin{bmatrix}w_1^2\\ w_2^2\end{bmatrix}$. Therefore, summation in the kernel domain corresponds to summation in the spectral domain, and multiplication in the kernel domain corresponds to convolution in the spectral domain.

For symmetric $S_1$ and $S_2$, their convolution
\begin{equation}
\begin{aligned}
    &(S_1 * S_2)(\bm{w_2}, \bm{w_1}) \\
    &= \int_{\bm{\tau}_1} \int_{\bm{\tau}_2} S_1(\bm{\tau}_2, \bm{\tau}_1) S_2(\bm{w}_2-\bm{\tau}_2, \bm{w}_1-\bm{\tau}_1) \mathrm{d}\bm{\tau}_1 \mathrm{d}\bm{\tau}_2 \\
    &= \int_{\bm{\tau}_1} \int_{\bm{\tau}_2} S_1(\bm{\tau}_1, \bm{\tau}_2) S_2(\bm{w}_1-\bm{\tau}_1, \bm{w}_2-\bm{\tau}_2) \mathrm{d}\bm{\tau}_1 \mathrm{d}\bm{\tau}_2  \\
    &= (S_1 * S_2)(\bm{w_1}, \bm{w_2})
\end{aligned}
\end{equation}
is still symmetric, and their sum is also symmetric. Therefore, the sum and multiplication of complex-valued non-stationary kernels is still a complex-valued non-stationary kernel. 

Finally, 
decompose $S(\bm{w}_1, \bm{w}_2)$ into even and odd parts $S(\bm{w}_1, \bm{w}_2)=E(\bm{w}_1, \bm{w}_2)+O(\bm{w}_1, \bm{w}_2)$, that $E(\bm{w}_1, \bm{w}_2) = E(-\bm{w}_1, -\bm{w}_2)$, $O(\bm{w}_1, \bm{w}_2)=-O(\bm{w}_2, \bm{w}_1)$. Inherently, $E$ and $O$ are symmetric. \\
Then the part corresponded with $E$ 
\begin{equation}
\begin{aligned}
    &\int e^{2\pi i (\bm{w}_1^\top\bm{x}-\bm{w}_2^\top \bm{y})} E(\bm{w}_1, \bm{w}_2) \mathrm{d}\bm{w}_1\mathrm{d}\bm{w}_2 \\
    &= \frac{1}{2}\int e^{2\pi i (\bm{w}_1^\top\bm{x}-\bm{w}_2^\top \bm{y})} S(\bm{w}_1, \bm{w}_2) \mathrm{d}\bm{w}_1\mathrm{d}\bm{w}_2\\
    &+\frac{1}{2}\int e^{2\pi i (\bm{w}_1^\top\bm{x}-\bm{w}_2^\top \bm{y})} S(-\bm{w}_1, -\bm{w}_2) \mathrm{d}\bm{w}_1\mathrm{d}\bm{w}_2 \\
    &=\frac{1}{2}\int [e^{2\pi i (\bm{w}_1^\top\bm{x}-\bm{w}_2^\top \bm{y})} + e^{2\pi i (-\bm{w}_1^\top\bm{x}+\bm{w}_2^\top \bm{y})}]\\
    &\quad \quad \quad  S(\bm{w}_1, \bm{w}_2) \mathrm{d}\bm{w}_1\mathrm{d}\bm{w}_2 \\
    &= \int \cos(2\pi (\bm{w}_1^\top\bm{x}-\bm{w}_2^\top \bm{y})) S(\bm{w}_1, \bm{w}_2) \mathrm{d}\bm{w}_1\mathrm{d}\bm{w}_2
\end{aligned}
\end{equation}
is real. The part corresponded with $O$
\begin{equation}
\begin{aligned}
    &\int e^{2\pi i (\bm{w}_1^\top\bm{x}-\bm{w}_2^\top \bm{y})} O(\bm{w}_1, \bm{w}_2) \mathrm{d}\bm{w}_1\mathrm{d}\bm{w}_2 \\
    &= \frac{1}{2}\int e^{2\pi i (\bm{w}_1^\top\bm{x}-\bm{w}_2^\top \bm{y})} S(\bm{w}_1, \bm{w}_2) \mathrm{d}\bm{w}_1\mathrm{d}\bm{w}_2\\
    &-\frac{1}{2}\int e^{2\pi i (\bm{w}_1^\top\bm{x}-\bm{w}_2^\top \bm{y})} S(-\bm{w}_1, -\bm{w}_2) \mathrm{d}\bm{w}_1\mathrm{d}\bm{w}_2 \\
    &=\frac{1}{2}\int [e^{2\pi i (\bm{w}_1^\top\bm{x}-\bm{w}_2^\top \bm{y})} - e^{2\pi i (-\bm{w}_1^\top\bm{x}+\bm{w}_2^\top \bm{y})}] \\ &\quad \quad \quad  S(\bm{w}_1, \bm{w}_2) \mathrm{d}\bm{w}_1\mathrm{d}\bm{w}_2 \\
    &= i \int \sin(2\pi (\bm{w}_1^\top\bm{x}-\bm{w}_2^\top \bm{y})) S(\bm{w}_1, \bm{w}_2) \mathrm{d}\bm{w}_1\mathrm{d}\bm{w}_2
\end{aligned}
\end{equation}
is purely imaginary. Therefore $E(\bm{w}_1, \bm{w}_2)$ corresponds to $\Re(k)$, the real part of kernel $k$. As $E$ is bounded, symmetric and positive, $\Re(k)$ is a valid real-valued kernel.
\end{proof}

\section{Proof for NKN Polynomial Universalities}
\label{app:nkn-poly-univ}
This section provides proof for Theorem~\ref{thm:poly-univ}. 

\begin{proof} 
\textit{For limited width}
For limited width, we will demonstrate a special stucture of NKN which can represent any positive-weighted polynomial of primitive kernels. As a polynomial is summation of several multiplication terms, if NKN can represent any multiplication term and add them iteratively to the outputs, it can represent any positive-weighted polynomial of primitive kernels.

Figure~\ref{fig:poly-flow} demonstrates an example of how to generate multiplication term $0.3k_1k_2^2$, in which $k_1, k_2, k_3$ represent primitive kernels and $1$ represents bias. In every Linear-Product module, we add a primitive kernel to the fourth neuron on the right and multiply it with the third neuron on the right. Iteratively, we can generate the multiplication term. Finally, we can add this multiplication term to the rightmost output. 

Note that, after adding the multiplication term to the output, we keep every kernel except the output kernel unchanged. Therefore, we can continue adding new multiplication terms to the output and generate the whole polynomial.

Although this structure demonstrates the example with 3 primitive kernel, it can easily extend to deal with any number of primitive kernels as can be seen. What's more, as a special case, this structure only uses specific edges of a NKN. In practice, NKN can be more flexible and efficient in representing kernels. 

\begin{figure*}
\centering
\includegraphics[width=0.85\linewidth]{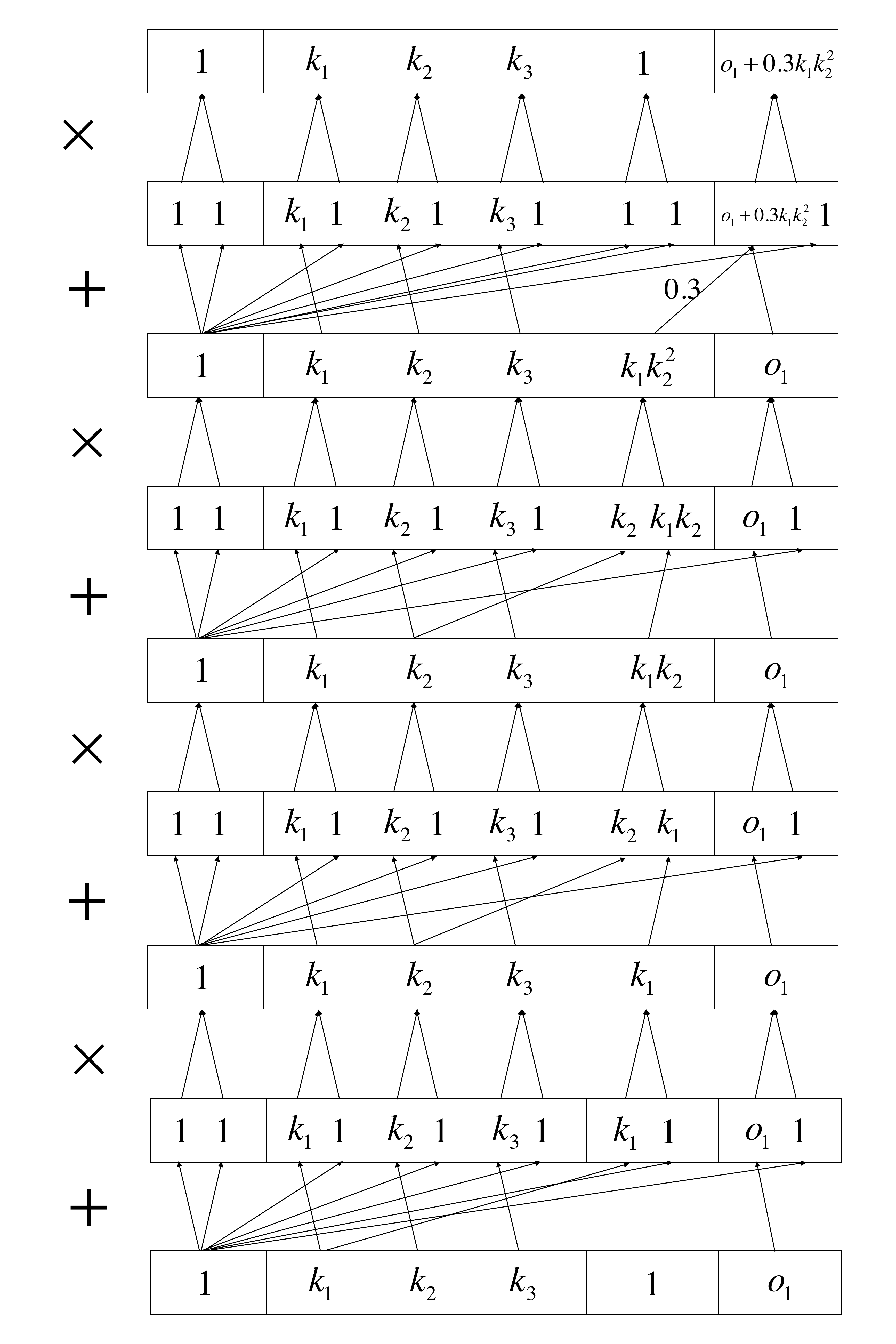}
\caption{A building block of NKN to represent any positive-weighed polynomial of primitive kernels. This structure shows how NKN can add a multiplication term to the output kernel as well as keeping primitive kernels unchanged.}
\label{fig:poly-flow}
\end{figure*}

\textit{For limited depth}, the example can be proved by recursion.

As a Lemma, we firstly prove, a NKN with one primitive kernel $k_0$ and $p$ Linear-Product modules, can use width $2(p+1)$ to represent $k_0^q, \forall q \le 2^p$.

It is obvious that NKN can use width 2 to represent $k_0^{2^t}, \forall t \le p$. Then we can write $q = \sum_{t=0}^{p} q_t 2^t$ with $q_t \in \{0, 1\}$. Therefore, in order to represent $k_0^q$, we can represent these terms separately, using width $2(p+1)$. Lemma proved.

Now we start the recusion proof,

If $B=1$, according to the lemma above, we can use width $2(p+1) \le 2^{Bp+1}$. The theorem also holds easily when $p=1$.

Now assume the statement holds for in cases less than $B$. We can write the polynomial according to the degree of $k_0$, which ranges from $0$ to $2^p$. According to the lemma, these degrees of $k_0$ can be represented with width $2(p+1)$. And according to the recusion assumption, the width we need is at most
$$
    2^{p-1} (2p + 2^{(B-1)(p-1)+1}) + (2^{p-1}+1) (2(p+1) + 2^{(B-1)p+1})
$$
Where the first term corresponds to degree of $k_0$ larger than $2^{p-1}$ and the second term corresponds to degree of $k_0$ no larger than $2^{p-1}$. The formula above is less than $2^{Bp+1}$. According to the recursion, the theorem is proved.

\end{proof}

\section{Proof for Example~\ref{examp: limitation-real}}
\label{app:limitation-real-kernel}
This section provides proof for Example~\ref{examp: limitation-real}.

\begin{proof} 
As all dimensions are independent, we only need to prove for the one-dimensional case. 

We consider these stationary kernels in spectral domain. Then cosine kernel $\cos(\mu \tau)$ corresponds to $\delta(w-\mu)+\delta(w+\mu)$ as spectral density. Note that summation and multiplication in kernel domain correspond to summation and convolution in spectral domain, respectively.

The example can be proved by recursion.

If $n=0$, the conclusion is obvious.

Now assume the conclusion holds when less than $n$.

Consider case $n$,

Then spectral density of target kernel $k^*$ will be $$f^* = \sum_{t=1}^{n+1} \binom{n}{2}^{2t}(\delta(w-4^t)+\delta(w+4^t))$$

Assume there exists a PWP $\bar{k}$ of cosine kernels $\{\cos(\mu_t \tau)\}_{t=1}^n$ that have 
$$
      \underset{\tau}{\max} |\bar{k}(\tau) - k^*(\tau)| < \epsilon
$$
Then we can find $\xi > 0$ such that
$$
    \underset{w}{\max} |\bar{f}(w) - f^*(w)| < \xi
$$
Because $\bar{k}$ is a PWP of cosine kernels. Its spectral density will be summations of convolutions of even delta funtions $\delta(w-\mu)+\delta(w+\mu)$. As convolution $\delta(w-w_1) \ast \delta(w-w_2) = \delta(w-(w_1+w_2))$, $\bar{f}$ can be represented as summations of delta functions, where location in every delta function has form $\sum_{t=1}^n p_t \mu_t$ with $p_t \in \mathbb{Z}$ and we use $\mathrm{coeff}(\sum_{t=1}^n p_t \mu_t)$ to represent the coefficient of this convolution term.

Denote the biggest location of $\bar{f}$ as $\sum_{t=1}^n p_{t} \mu_t$, then,
\begin{equation}
    \sum_{t=1}^n p_{t} \mu_t= 4^{n+1}
\end{equation}

If $\exists t_0$ such that $p_{t_0} < 0$, according to symmetry of $\delta(w-\mu)+\delta(w+\mu)$, there must be another term $$\sum_{t=1}^{t_0-1} p_{t}\mu_{t} - p_{t}\mu_{t} + \sum_{t=t_0+1}^{n} p_{t}\mu_{t}$$ which has the same coefficient in $\bar{f}$ as $ \sum_{t=1}^n p_{t} \mu_t$. However 
$$
\sum_{t=1}^{t_0-1} p_{jt}\mu_{jt} - p_{jt}\mu_{jt} + \sum_{t=t_0+1}^{n} p_{jt}\mu_{jt} > \sum_{t=1}^n p_{jt} = 4^{n+1}
$$ 
Contradiction!

If $\sum_{t=1}^n p_t \ge 3$, we can find the smallest $\mu$ and reverse its sign. For example, if the location is $2\mu_1 + \mu_2$, then according to symmetry, there must be another location $-\mu_1 + \mu_1 + \mu_2$ in f with the same coefficient. However,
$$
    -\mu_1 + \mu_1 + \mu_2 > \frac{1}{3}(\mu_1 + \mu_1 + \mu_2) = \frac{1}{3}4^{n+1} > 4^{n}
$$
Contradiction.

Therefore, $\sum_{t=1}^n p_t \le 2$. If one $\mu$ has degree 2, $2\mu=4^{n+1}$. According to symmetry, there will be a term with location $-\mu+\mu=0$, contradiction. If $\mu_{t_1} + \mu_{t_2} = 4^{n+1}$. Then there must be term $\mu_{t-1} - \mu_{t-2}$. However, the coefficient
$$
    \binom{n}{2}^{2n} >= \mathrm{coeff}(\mu_{t-1} - \mu_{t-2}) = \mathrm{coeff}(\mu_{t_1} + \mu_{t_2})
$$
While there have at most $2\binom{n}{2}$ 2-item terms, thus $$\mathrm{coeff}(\mu_{t_1} + \mu_{t_2}) >= \frac{\binom{n}{2}^{2(n+1)}}{2\binom{n}{2}} > \binom{n}{2}^{2n}$$
Contradiction!

Therefore, there exists $t$ such that $p_t=1, \mu_t=4^{n+1}$. Because $4^{n+1}$ is the biggest location in $f^*$, $\mu_t$ cannot appear in any other term of $\bar{f}$, which induces to case of $n-1$. According to induction, the statement is proved.

\end{proof}

\section{Proof of NKN's Stationary Universality}
\label{app:nkn-stationary-univ}
This section provides proof for Theorem~\ref{thm:stationary-univ}.

\begin{proof}
Start at 1-d primitive kernels.
For a given complex-valued stationary kernel $k$ with Fourier transform $g$, according to Lemma~\ref{lem:univesal-mog}, we can find $v \in \mathbb{R}, \mu \in \mathbb{R}^+$, such that $\exists v_q=n_q^1 v \in \{v, 2v, \cdots\}, \mu_q=n_q^2 \mu \in \{\pm \mu, \pm 2\mu, \cdots\}$,
\begin{equation}
\begin{aligned}
    \bar{g}(w) = \sum_{q=1}^Q \lambda_q \exp(-\frac{(w-\mu_q)^2}{2 v_q}) \\
    \| \bar{g}(w) - g(w) \|_1 \le \epsilon
\end{aligned}
\end{equation}
Thus the corresponding kernel 
\begin{equation}
\begin{aligned}
    \bar{k}(\tau) &= \sum_{q=1}^Q \lambda_q \exp(\frac{1}{2}\tau^2 v_q) e^{i\mu_q x} \\
    &= \sum_{q=1}^Q \lambda_q \exp(-\frac{1}{2}\tau^2 v)^{n_q^1} (e^{i\mu x})^{n_q^2} \\
    &=  \sum_{q=1, n_q^2 > 0}^Q \lambda_q \exp(-\frac{1}{2}\tau^2 v)^{n_q^1} (e^{i\mu x})^{|n_q^2|} \\
    &+ \sum_{q=1, n_q^2 < 0}^Q \lambda_q \exp(-\frac{1}{2}\tau^2 v)^{n_q^1} (e^{-i\mu x})^{|n_q^2|} \\
    &= \mathrm{Poly}_{+}(\exp(-\frac{1}{2}\tau^2 v), e^{i\mu x}, e^{-i\mu x})
\end{aligned} 
\end{equation}
 Here, $\mathrm{Poly}_{+}$ denotes a positive-weighted polynomial of primitive kernels. Therefore, any complex-valued stationary kernel can be well approximated by a positive-weighted polynomial of a $\mathrm{RBF}$ and two $e^{i\mu x}$. As Theorem~\ref{thm:poly-univ} shows, an NKN with limited width can approximate any positive-weighted polynomial of primitive kernels. By Wiener's Tauberian theorem \citep{wiener1932tauberian}, which states L1 convergence in Fourier domain implies pointwise convergence in original domain, we prove, taking  these three kernels as primitive kernels, an NKN with width no more than $2\times 3 + 6 = 12$ is dense towards any 1-dimensional complex-valued stationary kernel wrt pointwise convergence. To be noted, as $\cos$ is an even function, we only need two primitive kernels ($\mathrm{RBF}$, $e^{i\mu x}$) (width $10$) to approximate any real-valued stationary kernel.
 
For $d$ dimensional inputs, \citet{wilson2013gaussian} provides an SM kernel which is a universal approximator for all real-valued stationary kernels. 
\begin{equation}
    k(\bm{\tau}) = \sum_{q=1}^Q \lambda_q \exp(-\frac{1}{2} \bm{\tau}^\top \mathrm{diag}(\bm{v}_q) \bm{\tau}) \cos(\bm{\mu}_q^\top \bm{\tau})
\end{equation}
For complex-valued stationary kernels corresponding to 
\begin{equation}
\begin{aligned}
    k(\bm{\tau}) &= \sum_{q=1}^Q \lambda_q \exp(-\frac{1}{2} \bm{\tau}^\top \mathrm{diag}(\bm{v}_q) \bm{\tau}) e^{i\bm{\mu}_q^\top \bm{\tau}} \\
    &= \sum_{q=1}^Q \lambda_q \prod_{j=1}^d \exp(-\frac{1}{2} \tau_j^2 v_{qj}) e^{i\mu_{qj} \tau_j},
\end{aligned}
\end{equation}
the formula above is a positive-weighted polynomial of 1-dimensional case. Similarly, we can approximate them well using an NKN with $d$ RBF and $2d$ $e^{i\bm{\mu}^\top \bm{x}}$ as primitive kernels and with width no more than $6d+6$. For a real-valued stationary kernel, we only need an NKN with $d$ RBF and $d$ $e^{i\bm{\mu}^\top \bm{x}}$ as primitive kernels and with width no more than $4d+6$.
\end{proof}

\section{NKN on Approximating Bounded Kernels}\label{app:nkn-closure}
\begin{thm}\label{thm:nkn-nonstationary-univ} 
By using $d$ RBF kernels, $2d$ exponential kernels $e^{i(\bm{x}-\bm{y})^\top \bm{\mu}}$ and $d$ functions $e^{i(\bm{x}+\bm{y})^\top \bm{\nu}}$ that only support multiplication and always apply the same operation on $e^{-i(\bm{x}+\bm{y})^\top \bm{\nu}}$, NKN with width no more than $8d+6$ can approximate all $d$-dimensional generalized spectral kernels in \citet{kom2015generalized}.
\end{thm}

\begin{proof}
This proof refers to the proof in \citet{kom2015generalized}.

Let $(\bm{x}, \bm{y}) \rightarrow k^*(\bm{x},\bm{y})$ be a real-valued positive semi-definite, continuous, and integrable function such that $\forall \bm{x}, \bm{y}, k^*(\bm{x},\bm{y})>0$.
 
$k^*$ being integrable, it admits a Fourier transform
$$
    K^*(\bm{w}_1, \bm{w}_2) = \mathcal{F}(k^*)(\bm{w}_1, \bm{w}_2), 
$$
therefore, 
$$
    \forall \bm{x},\bm{y}, \quad \mathcal{F}(K^*)(\bm{x}, \bm{y}) = k^*(-\bm{x}, -\bm{y}) > 0.
$$
Hence $k^*$ suffices Wiener's tauberian theorem, so that any integrable function on $\mathbb{R}^d \times \mathbb{R}^b$ can be approximated well with linear combinations of translations of $K^*$.

Let $k$ be any continuous bounded kernel\footnote{\citet{yaglom2012correlation} points out that there are some exception cases that cannot be written as this, but these are specially designed and hardly ever seen in practice.}, according to Theorem~\ref{thm:nonstationary-bochners},
$$
    k(\bm{x}, \bm{y}) = \int e^{i(\bm{x}^\top\bm{w}_1 - \bm{y}^\top \bm{w}_2)} f(\bm{w}_1, \bm{w}_2)\mathrm{d}\bm{w}_1\mathrm{d}\bm{w}_2.
$$
Note that we can approximate $f$ as 
$$
    f(\bm{w}_1, \bm{w}_2) \approx \sum_{k=1}^K \lambda_k K^*(\bm{w}_1+\bm{w}_k^1, \bm{w}_2+\bm{w}_k^2).
$$
Therefore, we have an approximator for $k(\bm{x},\bm{y})$:
\begin{equation}
\begin{aligned}
    &\hat{k}(\bm{x}, \bm{y})\\
     &= \int e^{
     i\begin{bmatrix}
     	\bm{x} \\ -\bm{y}
     \end{bmatrix}^\top 
     \begin{bmatrix}
     	\bm{w}_1 \\ \bm{w}_2
     \end{bmatrix}} \sum_{k=1}^K \lambda_k K^*(\bm{w}_1+\bm{w}_k^1, \bm{w}_2+\bm{w}_k^2) \mathrm{d}\bm{w} \\
    &= \sum_{k=1}^K \lambda_k k^*(\bm{x}, \bm{y}) e^{i(\bm{x}^\top\bm{w}_k^1 - \bm{y}^\top\bm{w}_k^2).}
\end{aligned}
\end{equation}
Where $\bm{w}=\begin{bmatrix}
	\bm{w_1} \\ \bm{w_2}
\end{bmatrix}$. Plus, for $k(\bm{x}, \bm{y})$ being a kernel, we must have $k(\bm{x}, \bm{y})=k(\bm{y}, \bm{x})*$, corresponding to $f(\bm{w}_1, \bm{w}_2) = f(\bm{w}_2, \bm{w}_1)$, therefore we have the universal approximator 
\begin{equation}
\begin{aligned}
    \bar{k}(\bm{x}, \bm{y}) &= \sum_{k=1}^K \lambda_k k^*(\bm{x}, \bm{y}) (e^{i(\bm{x}^\top\bm{w}_k^1 - \bm{y}^\top\bm{w}_k^2)}+e^{i(\bm{x}^\top\bm{w}_k^2 - \bm{y}^\top\bm{w}_k^1)}) \\
    &= \sum_{k=1}^K \lambda_k k^*(\bm{x}, \bm{y})e^{i(\bm{x}-\bm{y})^\top \frac{\bm{w}_k^1 + \bm{w}_k^2}{2}}  \\
    &\quad (e^{i(\bm{x}+\bm{y})^\top \frac{\bm{w}_k^1 - \bm{w}_k^2}{2}}+e^{-i(\bm{x}+\bm{y})^\top \frac{\bm{w}_k^1 - \bm{w}_k^2}{2}}) \\
    &= \sum_{k=1}^K \lambda_k k^*(\bm{x}, \bm{y})e^{i(\bm{x}-\bm{y})^\top \bm{\mu}_k} (e^{i(\bm{x}+\bm{y})^\top \bm{\nu}_k}+e^{-i(\bm{x}+\bm{y})^\top \bm{\nu}_k}) \\
    &= \sum_{k=1}^K \lambda_k k^*(\bm{x}, \bm{y})(e^{i(\bm{x}-\bm{y})^\top \bm{\mu}_0})^{n_k} \Re[(e^{i(\bm{x}+\bm{y})^\top \bm{\nu}_0})^{m_k}]
\end{aligned}
\end{equation}
where $\lambda \in \mathbb{R}$ and for $k^*(\bm{x},\bm{y})$ we can use RBF. However, not every $\lambda$ makes $\bar{k}$ a kernel; therefore, for practical use, we limit $\lambda \in \mathbb{R}^+$. If we only focus on one part of $\bar{k}(\bm{x}, \bm{y})$, that is 
\begin{equation}
\begin{aligned}
\bar{k}_h(\bm{x}, \bm{y}) &= \sum_{k=1}^K \lambda_k k^*(\bm{x}, \bm{y})(e^{i(\bm{x}-\bm{y})^\top \bm{\mu}_0})^{n_k} (e^{i(\bm{x}+\bm{y})^\top \bm{\nu}_0})^{m_k} \\
    &= \mathrm{Poly}_+ \{k^*(\bm{x}, \bm{y}), e^{i(\bm{x}-\bm{y})^\top \bm{\mu}_0}, e^{-i(\bm{x}-\bm{y})^\top \bm{\mu}_0}, \\
    & \quad\quad\quad\quad\quad e^{i(\bm{x}+\bm{y})^\top \bm{\nu}_0}, e^{-i(\bm{x}+\bm{y})^\top \bm{\nu}_0} \}
\end{aligned}
\end{equation}
according to \ref{thm:poly-univ}, $\bar{k}_h(\bm{x}, \bm{y})$ can be represented with an NKN. In order to generate $\bar{k}(\bm{x}, \bm{y})$, we only increment a class with two items $e^{i(\bm{x}+\bm{y})^\top \bm{\nu}_0}, e^{-i(\bm{x}+\bm{y})^\top \bm{\nu}_0}$ and apply the same transformation to these two items each time. In order to keep the output of NKN as a kernel, we can limit this class to only be enabled for multiplication.

\end{proof}

\section{Discrete Gaussian Mixture Approximation}
\label{app:mog-univ}
\begin{lem}\label{lem:univesal-mog}
Gaussian mixtures with form 
\begin{equation}
    G(x) = \sum_{q=1}^Q w_q \exp(-\frac{(x-\psi_q)^2}{2f_q})
\end{equation}
with $Q \in \mathbb{Z}^+, w_q, \psi, f \in \mathbb{R}^+$,  and $\psi_q$ and $f_q$ selected discretely in $\mathcal{S}_{\psi}=\{\psi, \pm 2\psi, \cdots\}, \mathcal{S}_{f}=\{f, 2f, \cdots\}$ is dense with respect to $L_1$ convergence.
\end{lem}
\begin{proof}

For any continuous positive function $g(x)$ and $\epsilon > 0$, the Gaussian mixture approximation theorem ensures existing
\begin{equation}
\hat{g}(x) = \sum_{q=1}^Q w_q \exp(-\frac{ (x-\mu_q)^2 }{2v_q} )
\end{equation}
that $\| g(x)-\hat{g}(x)\|_1 < \frac{\epsilon}{2}$, \\
According to Pinsker's Inequality \citep{csiszar2011information}, for two probability distributions $P, Q$,
\begin{equation}
\begin{aligned}
    \|P-Q\|_1 \le \sqrt[]{\frac{1}{2} \mathrm{D}_{KL}(P\| Q)}
\end{aligned}
\end{equation}
For small enough $\mathrm{d}v$ and $\mathrm{d}\mu$, we have 
\begin{equation}
\begin{aligned}
& \mathrm{D}_{KL}(\mathcal{N}(\mu+\mathrm{d}\mu, v+\mathrm{d}v)\| \mathcal{N}(\mu, v)) \\ 
&= -\frac{1}{2} \log \frac{v+\mathrm{d}v}{v} + \frac{v+\mathrm{d}v + (\mathrm{d}\mu)^2}{2v} - \frac{1}{2} \\
& = -\frac{1}{2} \log \frac{v+\mathrm{d}v}{v} + \frac{\mathrm{d}v + (\mathrm{d}\mu)^2}{2v} \le \frac{1}{2\pi v} \frac{\epsilon^2}{2w_q^2Q^2}
\end{aligned}
\end{equation}
Thus, the $L_1$ norm between two gaussians is bounded by
\begin{equation}
\|\mathcal{N}(\mu+\mathrm{d}\mu, v+\mathrm{d}v) - \mathcal{N}(\mu, v)\|_1 \le \frac{1}{\sqrt[]{2\pi v}}\frac{\epsilon}{2|w_q|Q}
\end{equation}

Take $f, \psi$ small enough such that $\forall q=1,\cdots Q$
$$
\|\mathcal{N}(\mu_q+\frac{\psi}{2}, v_q+\frac{f}{2}) - \mathcal{N}(\mu_q, v_q)\|_1 \le \frac{1}{\sqrt[]{2\pi v_q}} \frac{\epsilon}{2|w_q|Q}
$$
Then for any $v_i, \mu_i$, we can find $f_i \in F=\{f, 2 f, \cdots \}, \psi_i \in \Psi=\{\psi, \pm 2\psi, \cdots \}$ such that $|v_i-f_i| < \frac{f}{2}, |\mu_i-\psi_i| < \frac{\psi}{2}$, define discrete Gaussian mixture
\begin{equation}
\begin{aligned}
	\bar{g}(x) \vcentcolon= \sum_{q=1}^Q w_q \sqrt[]{\frac{v_q}{f_q}} \exp(-\frac{(x-\psi_q)^2}{2f_q}) 
\end{aligned}
\end{equation}
According to triangular inequality, 
\begin{equation}
\begin{aligned}	
    &\|\bar{g}(x) - g(x)\|_1 \leq \|\bar{g}(x) - \hat{g}(x)\|_1 + \|\hat{g}(x) - g(x)\|_1 \\
    &< \| \sum_{q=1}^Q w_q \sqrt[]{2\pi v_q} [\mathcal{N}(x| \phi_q, f_q)- \mathcal{N}(x| \mu_q, v_q)] \|_1 + \frac{\epsilon}{2} \\
    &\le \sum_{q=1}^Q |w_q|\sqrt[]{2\pi v_q} \|\mathcal{N}(x| \phi_q, f_q)- \mathcal{N}(x| \mu_q, v_q)\|_1  + \frac{\epsilon}{2} \\
    &\le \sum_{q=1}^Q |w_q| \frac{\epsilon}{2|w_q|Q} + \frac{\epsilon}{2}  = \epsilon
\end{aligned}
\end{equation}
\end{proof}

\section{Additional Results}

\subsection{Synthetic Data}
\begin{figure}[h!]
\centering
    \includegraphics[width=\columnwidth]{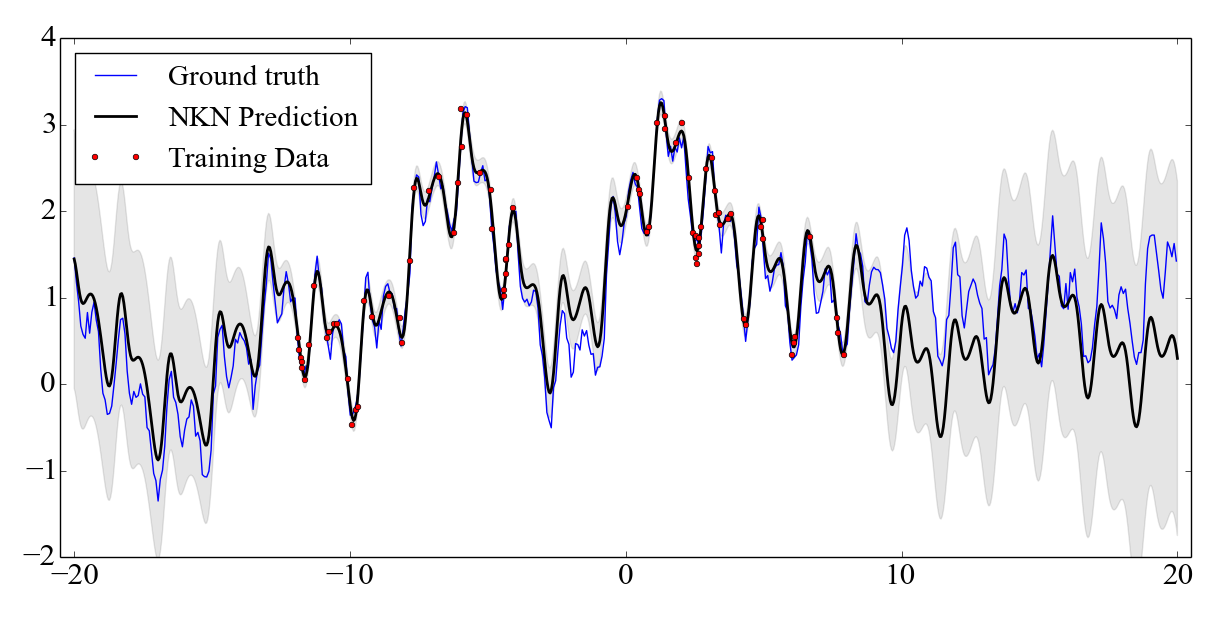}
\vspace{-2em}
\caption{Synthesized $1\text{-}d$ toy experiment. We compare NKN prediction with the ground truth function. Shaded area represents the standard deviation of NKN prediction.}
\label{fig:nkn_toy_syn}
\end{figure}


We first considered a synthetic regression dataset for illustrative purposes. Specifically, we randomly sampled a function from a GP with mean $0$ and a sum of RBF and periodic kernels. We randomly selected $100$ training points from $[-12, 0] \cup [6, 14]$.

As shown in Figure~\ref{fig:toy_as}, the NKN is able to fit the training data, and its extrapolations are consistent with the ground truth function. We also observe that the NKN assigns small predictive variance near the training data and larger predictive variance far away from it.
%

\subsection{1D Time-Series}

Figure~\ref{fig:toy_as_mauna} and Figure~\ref{fig:toy_as_solar} plot the 1D Time-Series extrapolations result, along with Figure~\ref{fig:toy_as}. For these experiments, the running time of AS and NKN are shown in \ref{tab:sec-as}.

\begin{figure}[h!] 
\centering
\includegraphics[width=0.95\columnwidth]{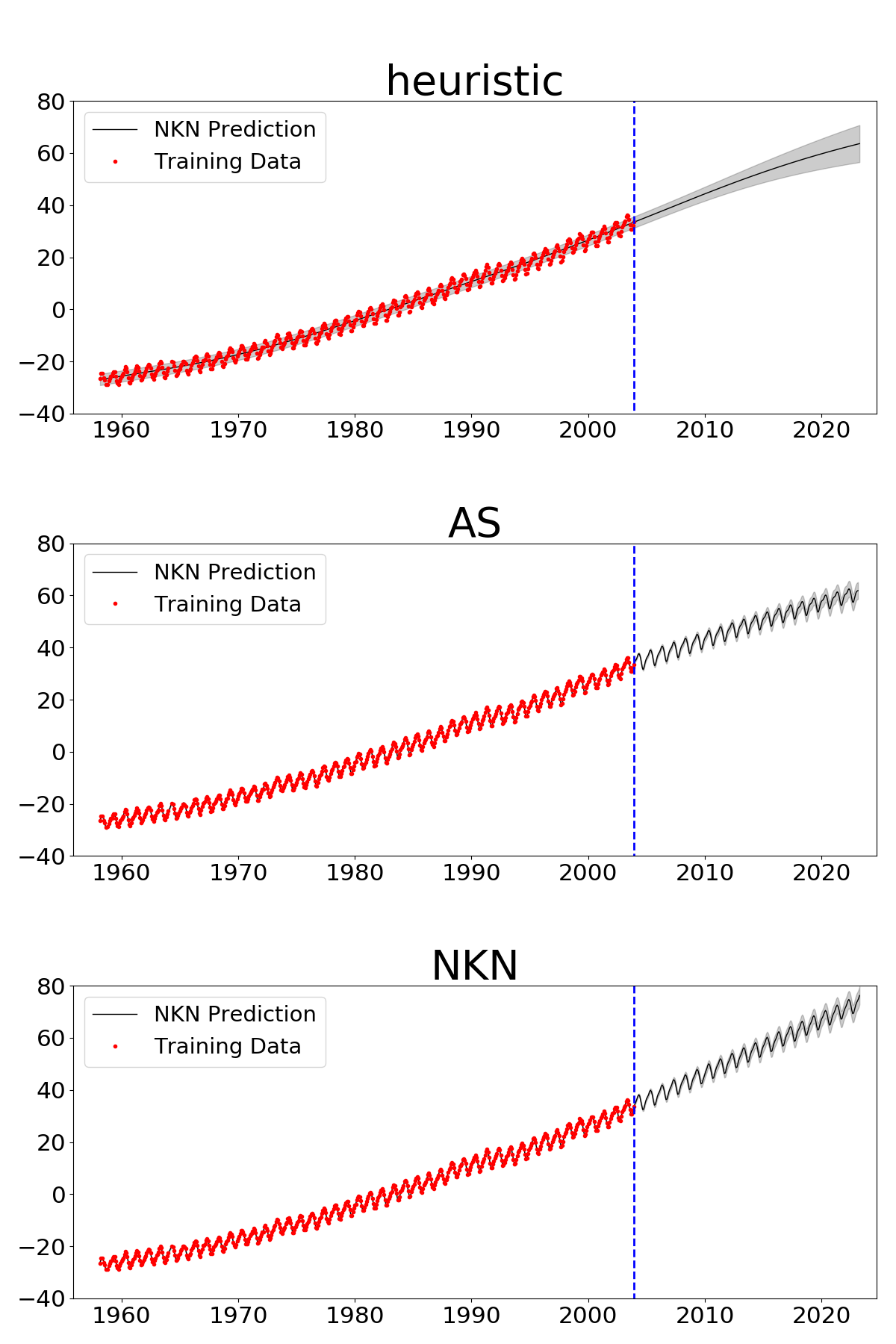}
\caption{Extrapolation results of NKN on the Mauna datasets. Heuristic kernel represents linear combination of RBF, PER, LIN, and Constant kernels. AS represents Automatic Statistician \citep{duvenaud2013structure}.
The \textcolor{red}{\bf red circles} are the training points, and the curve after the \textcolor{blue}{\bf blue dashed line} is the extrapolation result. Shaded areas represent 1 standard deviation.}
\label{fig:toy_as_mauna}
\end{figure}

\begin{figure}[h!] 
\centering
\includegraphics[width=0.95\columnwidth]{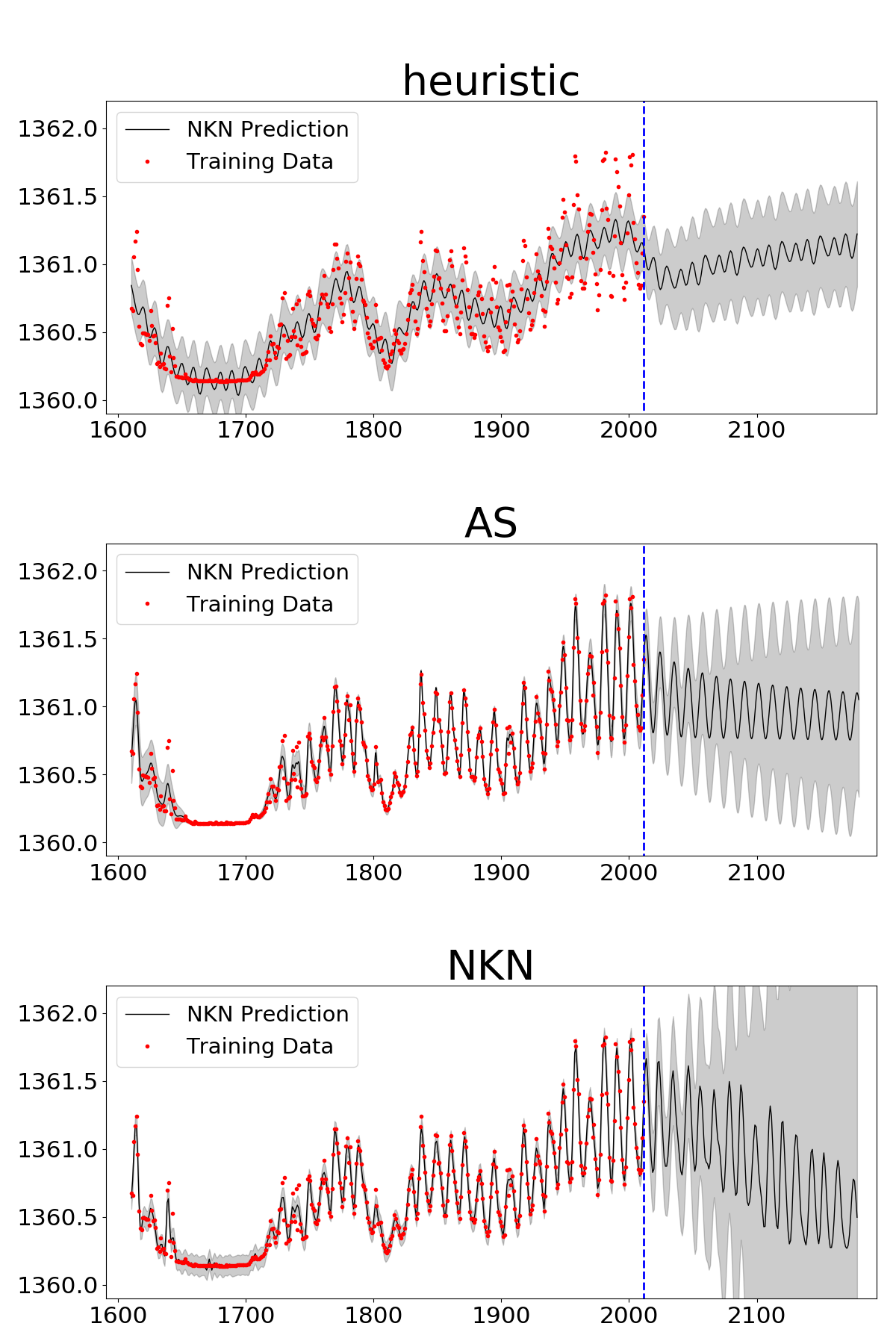}
\caption{Extrapolation results of NKN on the Solar datasets. Heuristic kernel represents linear combination of RBF, PER, LIN, and Constant kernels. AS represents Automatic Statistician \citep{duvenaud2013structure}.
The \textcolor{red}{\bf red circles} are the training points, and the curve after the \textcolor{blue}{\bf blue dashed line} is the extrapolation result. Shaded areas represent 1 standard deviation.}
\label{fig:toy_as_solar}
\end{figure}

\begin{table}[t]
\centering
\caption{Time (seconds) in 1D Time-series experiments.}
\label{tab:sec-as}
\begin{center}
\resizebox{0.8\columnwidth}{!}{
\begin{sc}
\begin{tabular}{cccc}
\toprule
Datasets   & Airline & Mauna & Solar \\
\midrule
AS  & 6147    & 51065 & 37716 \\
NKN & 201     & 576   & 962  \\
\bottomrule
\end{tabular}
\end{sc}
}
\end{center}
\end{table}
%
%
\subsection{Neuron Pattern Analysis}

We perform another synthetic toy experiment to analyze individual neurons' prediction patterns in NKN. The kernel for generating random functions is $\text{LIN}+\text{RBF} * \text{PER}$. We use NKN with LIN, PER, and RBF as primitive kernels and the following layers organized as Linear4--Product2--Linear1. The prediction result is shown in Figure~\ref{app:fig:nkn_toy_syn}. As we can see, NKN fits the data well with plausible patterns. Beyond that, we also visualize the pattern learned by each neuron in first layer of NKN in Figure~\ref{app:fig:toy_syn_neurons}.

As shown in Figure~\ref{app:fig:toy_syn_neurons}, these four neurons show different patterns in terms of frequencies. The third neuron learns the highest frequency and the first neuron learns the lowest frequency. This shows that NKN can automatically learn basis with low correlations. 

The question of how to steer the NKN framework to generate more interpretable patterns is an interesting subject for future investigation.

\begin{figure}[h!]
\centering
    \includegraphics[width=\columnwidth]{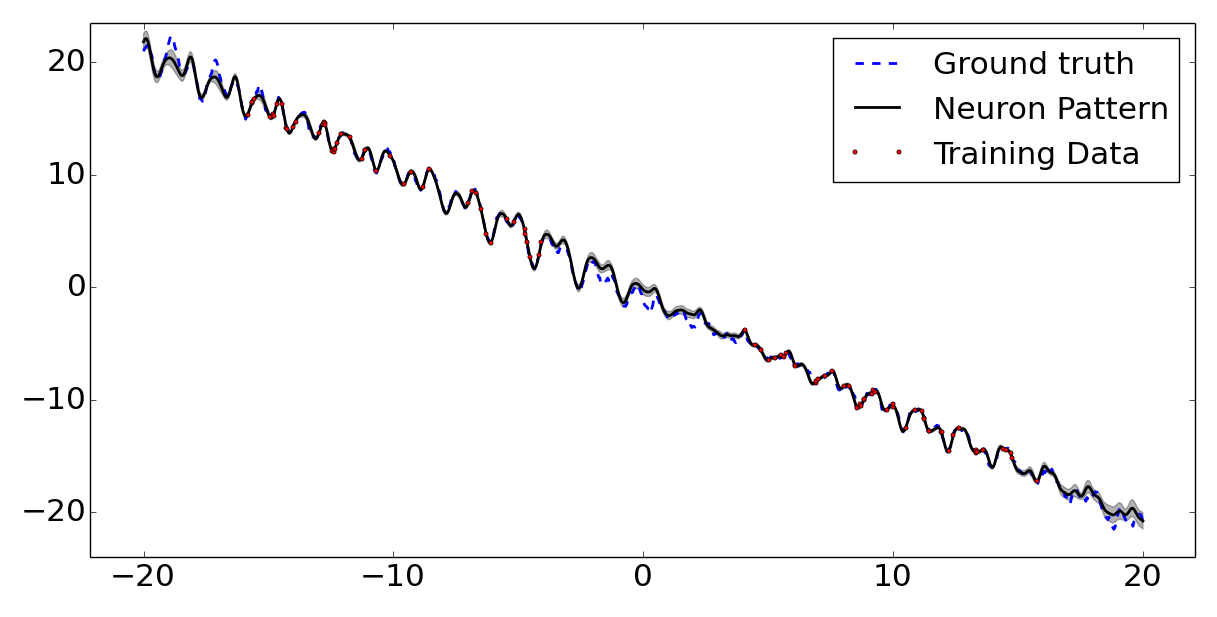}
\vspace{-2em}
\caption{Synthesized $1\text{-}d$ toy experiment.}
\label{app:fig:nkn_toy_syn}
\end{figure}

\begin{figure*}[h!] 
\centering
\hspace{-1em}%
\subfigure[The first neuron] { \label{fig:toy_neuron1} 
\includegraphics[height=0.17\textwidth]{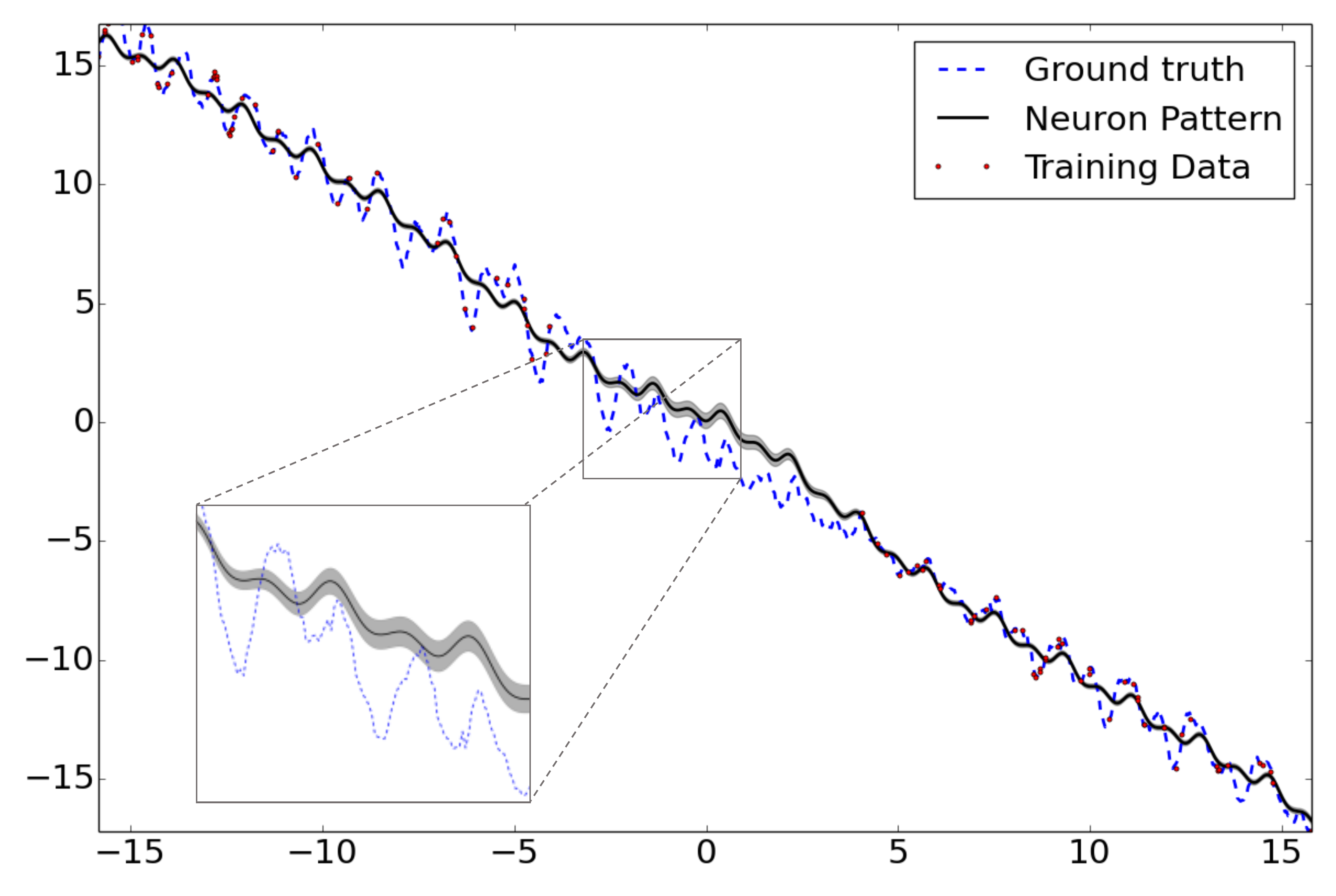} 
} 
\hspace{-1em}
\subfigure[The second neuron] { \label{fig:toy_neuron2} 
\includegraphics[height=0.17\textwidth]{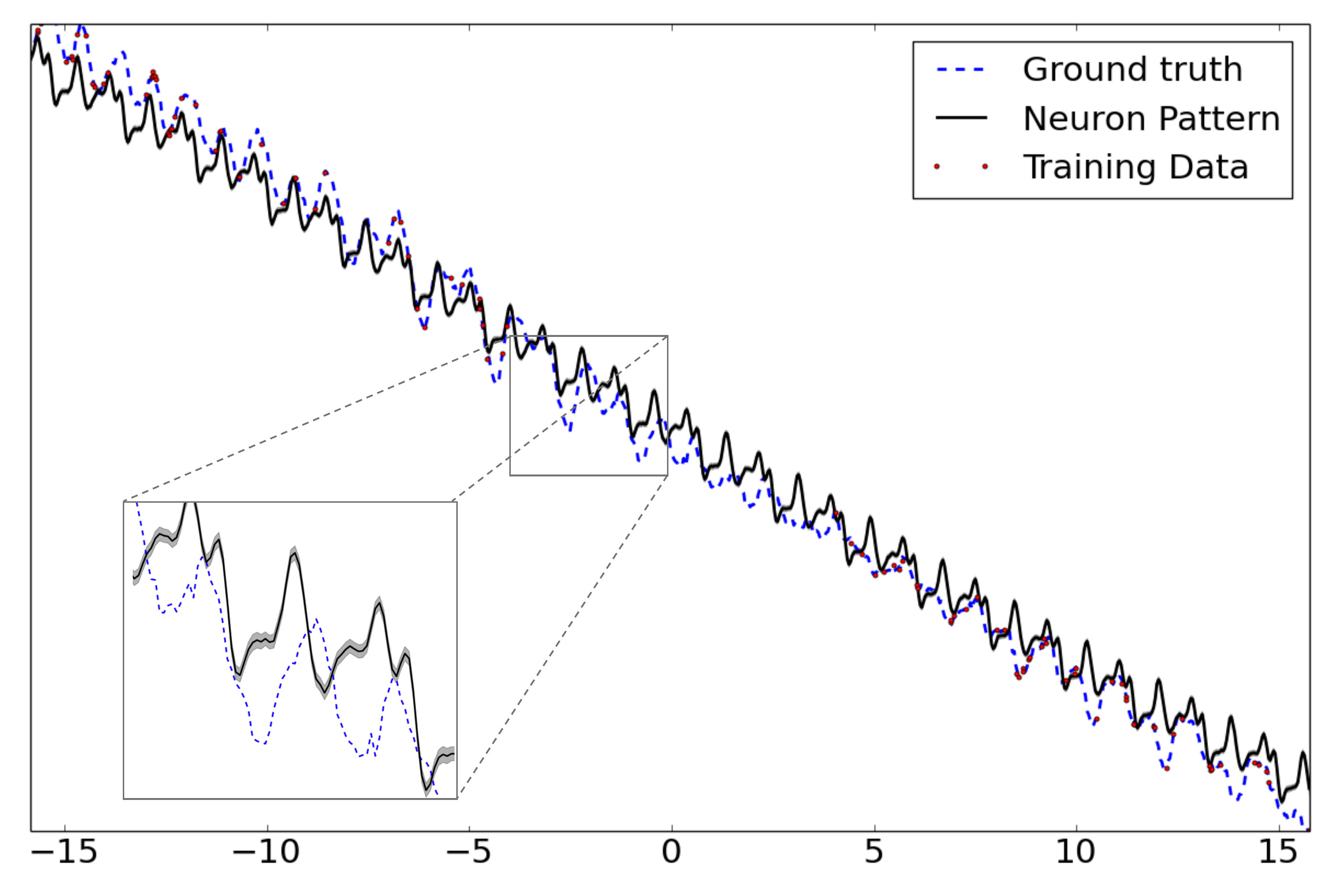} 
}
\hspace{-1em}%
\subfigure[The third neuron] { \label{fig:toy_neuron3} 
\includegraphics[height=0.17\textwidth]{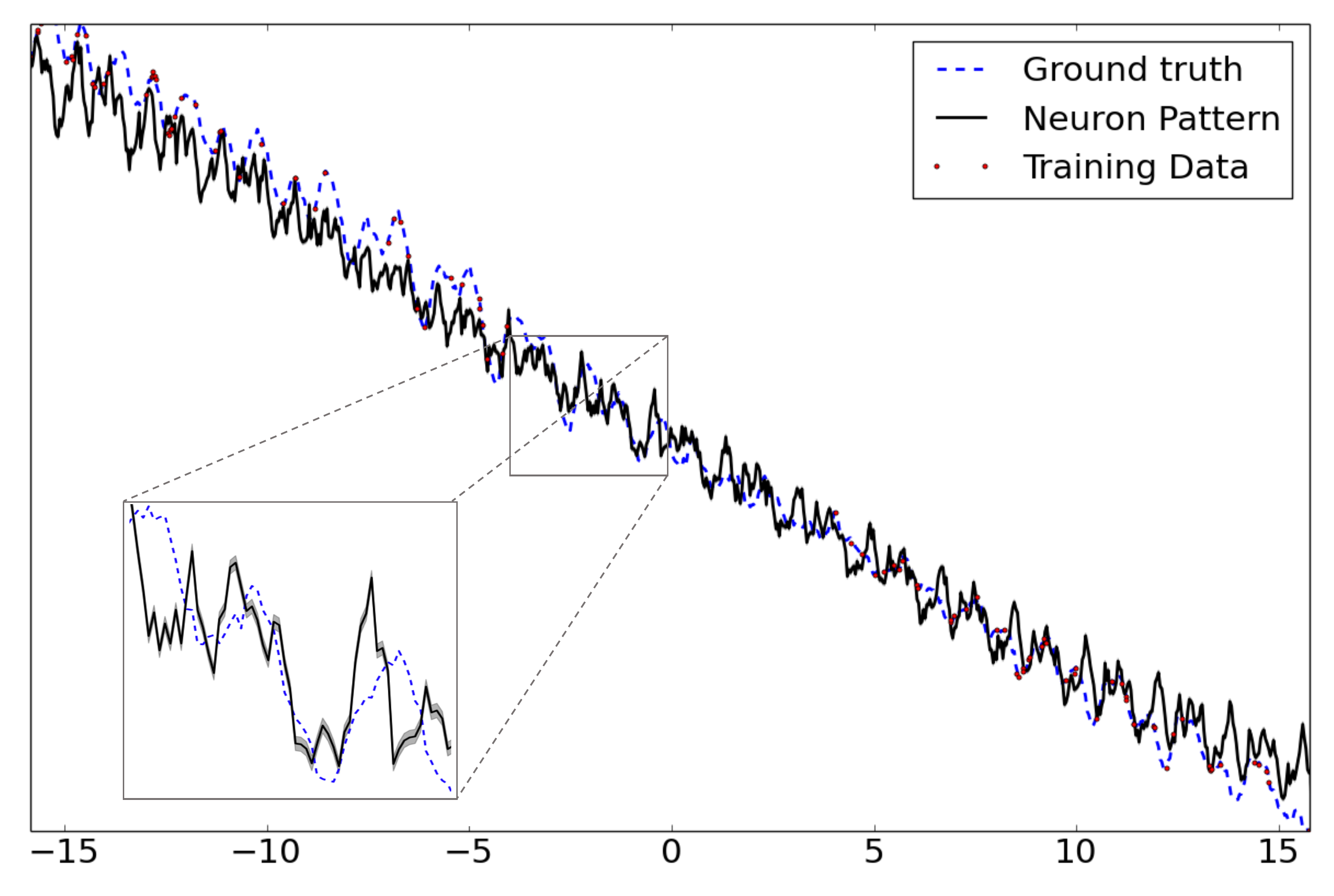} 
}
\hspace{-1em}%
\subfigure[The fourth neuron] { \label{fig:toy_neuron4} 
\includegraphics[height=0.17\textwidth]{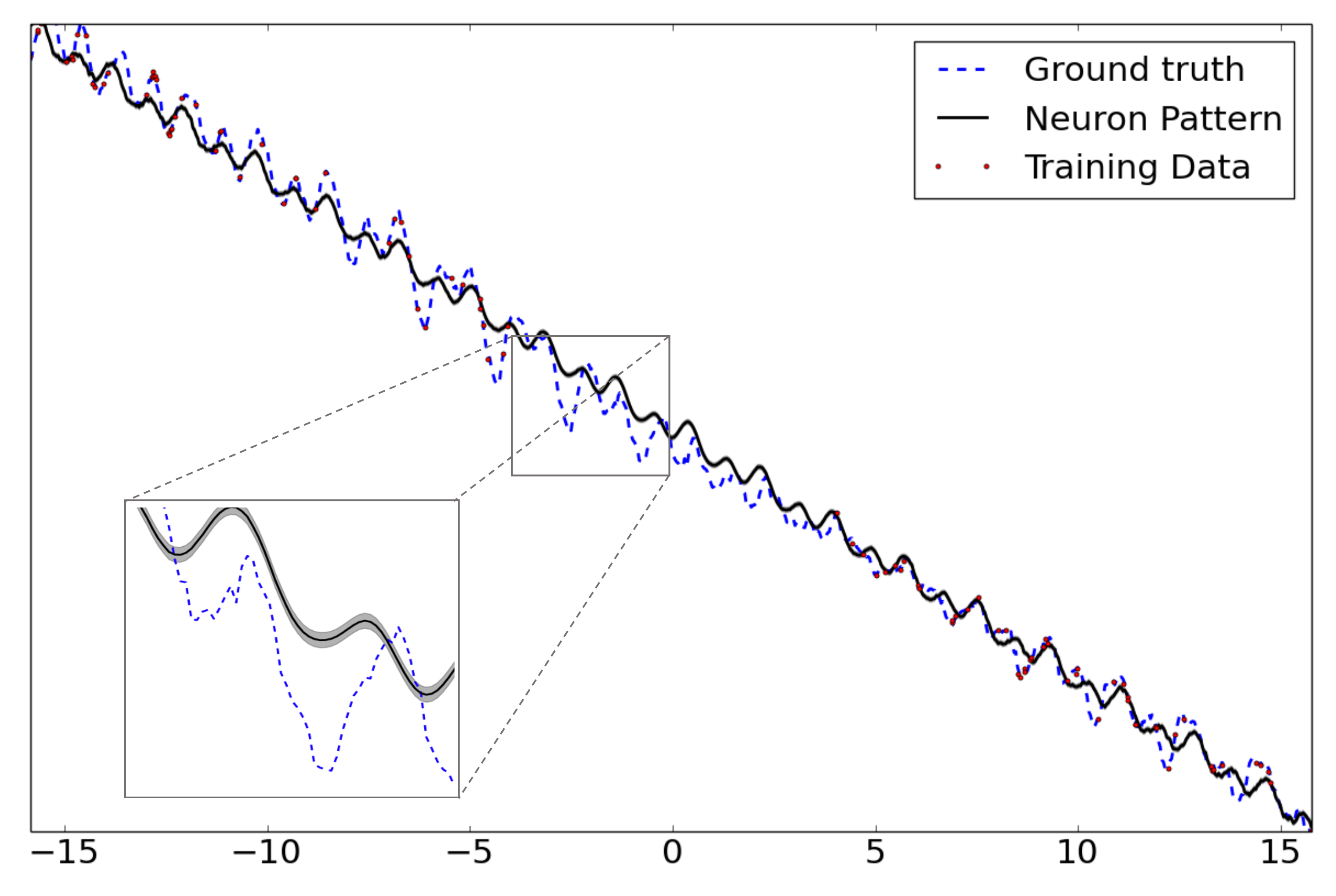} 
}  
\vspace{-1em}
\caption{The visualization of learned pattern of each neuron in the first layer of NKN.}\label{app:fig:toy_syn_neurons}
\label{app:fig:neurons} 
\end{figure*}

\subsection{Kernel Recovery and Interpretability}

As we proved in the paper theoretically, NKN is capable to represent sophisticated combinations of primitive kernels. We perform another experiment analyzing this kernel recovery ability in practice. 

The experiment settings are the same as Sec~\ref{subsec:bo}, that we use individual dimensional RBF as primitive kernels of a NKN to fit the additive function. We set input dimension as $10$ and use $100$ data points. Ideally, a well-performing model will learn the final kernel as summation of these additive groups. After training NKN for 20,000 iterations, we print 20 terms with biggest coefficients of the overall kernel polynomial. Here $k0, k1 \cdots, k9$ represents the 10 primitive RBF kernels, respectively.

For fully additive Stybtang function, the 20 biggest terms in the final NKN kernel polynomial is 

1, k6, k9, k5, k1, k2, k3, k0, k8, k4, k7, k2*k8, k6*k9, k5*k6, k1*k6, k5*k9, k1*k9, k1*k5, k6**2, k3*k9


For group additive Stybtang-transformed function with $k0,k1,k2,k3,k4$ as one group and $k5,k6,k7,k8,k9$ as one group, the 20 biggest terms in the final NKN kernel polynomial is 

k1*k3*k4*k7, k2*k5*k7*k8, k5*k7*k8*k9, k1*k3*k4**2, k5*k6*k7*k8, k3**2*k4*k7, k1**2*k4*k7, k2*k7**2*k9, k2*k6*k7**2, k3**2*k4**2, k1**2*k4**2, k6*k7**2*k9, k1*k3*k7**2, k5**2*k8**2, k2**2*k7**2, k3*k4*k7**2, k1*k4*k7**2, k7**2*k9**2, k3**2*k7**2, k3*k4**2*k7


We can see, for Stybtang function, the biggest ones are exactly all the linear terms. This shows NKN learns the additive structure from 100 data points in this 10 dimensional space. For group additive Stybtang-transformed function, except $k7$ appears in all terms, the kernels within the same group basically appears in the same term. This demonstrates again that NKN recovers the underlying data structure.

Michalewicz function is more complicated with steep valleys. The final polynomial with $100$ data points didn't show clear patterns. Therefore we change to model it with $500$ data points. The $20$ biggest terms are shown below,

k2*k4, k0*k9, k2*k8, k4*k7, k2*k3, k1, k3*k7, k7*k8, k0*k7, k3*k4, k6, k9**2, k7, k0*k2, k3*k8, k3*k9, k3**2, k2, k1*k5, k5*k6


We can see, although the biggest terms are not linear terms, all of them are either linear terms or quadratic terms. Note that for this NKN architecture with 2 Linear-Product modules, most common terms are fourth power. Therefore, this polynomial can show that this function is of low-correlation across dimensions to some degree.

Note that although NKN can produce sensible extrapolations, it is less interpretable compared to AS. However, as shown above, by inspecting the final kernel polynomial, we can indeed find interpretable information about the data distribution. Probably, combing with some clustering algorithms, this information recovery can become more ‘Automatic’. What's more, NKN's fast speed makes it possible to try different primitive kernel and network structure configurations, which might offer interpretable information from different aspects. This can be an interesting future research topic.

\subsection{Bayesian Optimization}
\label{subsec:bo-app}

\begin{figure*}[t]
\centering
\hspace{-1em}%
\subfigure[Stybtang] { \label{fig:bo_sty} 
\includegraphics[height=4.0cm]{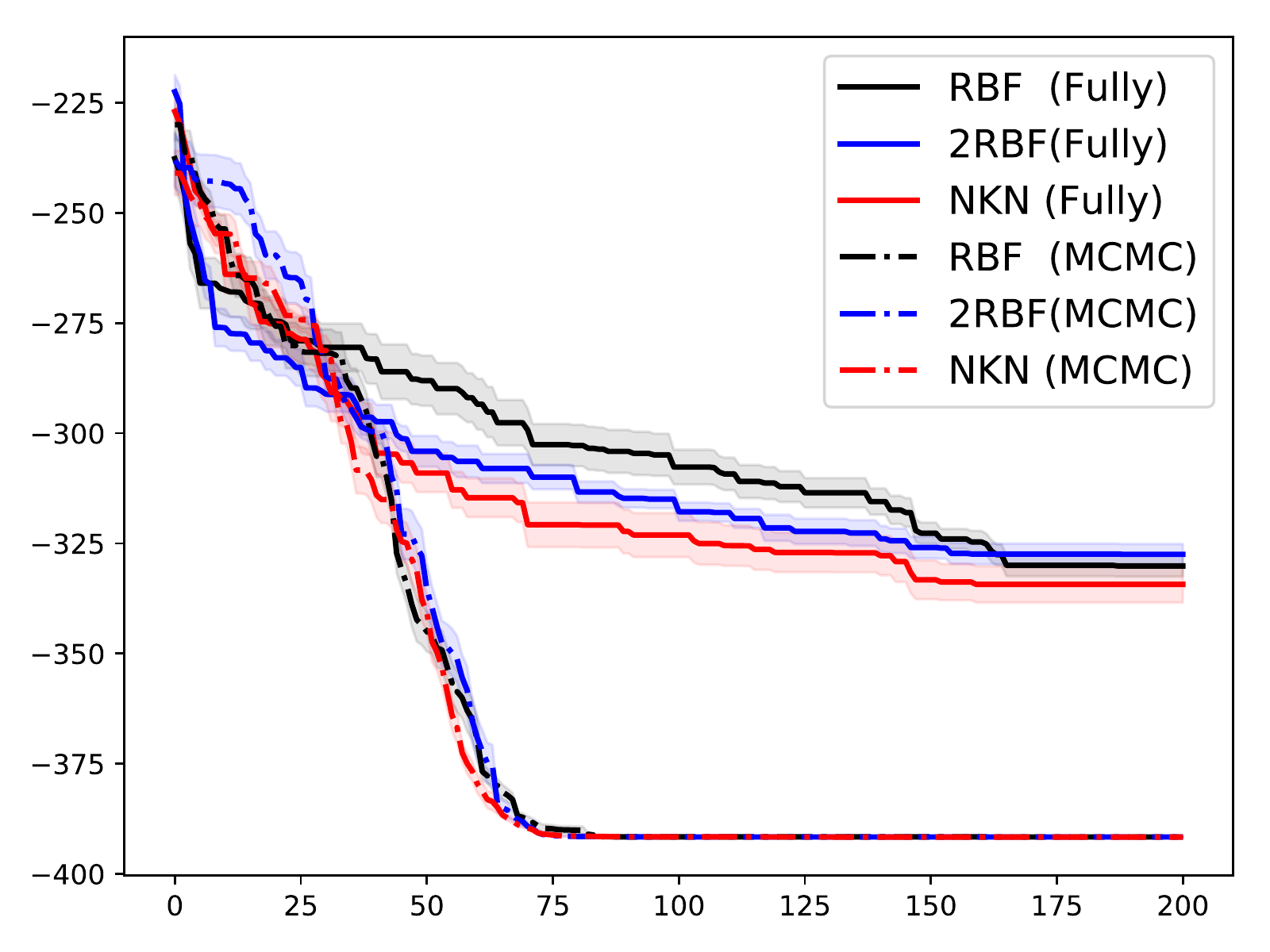} 
} 
\hspace{-1em}
\subfigure[Michalewicz] { \label{fig:bo_mich} 
\includegraphics[height=4.0cm]{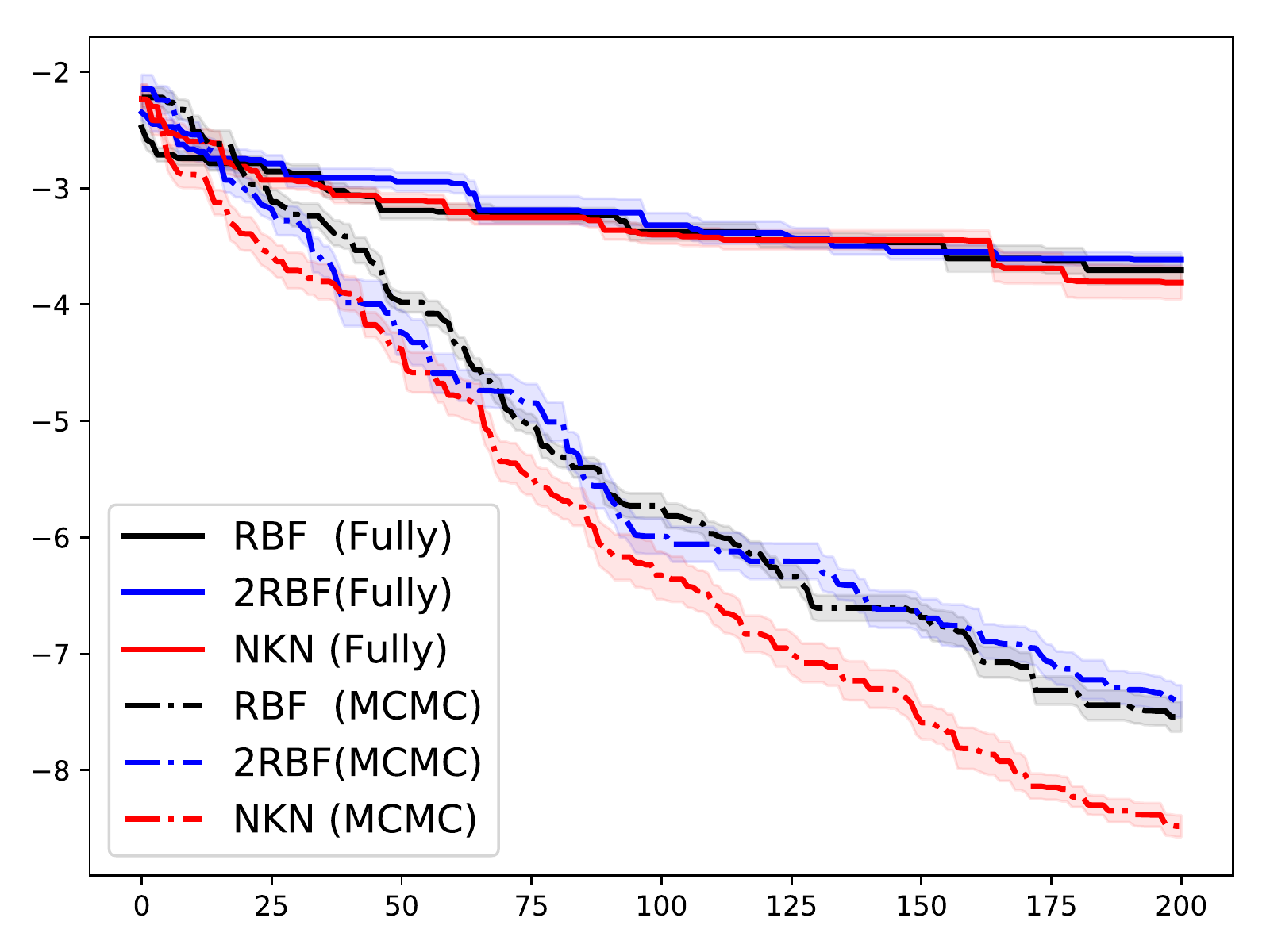} 
}
\hspace{-1em}%
\subfigure[Stybtang-transform] { \label{fig:bo_sty_t} 
\includegraphics[height=4.0cm]{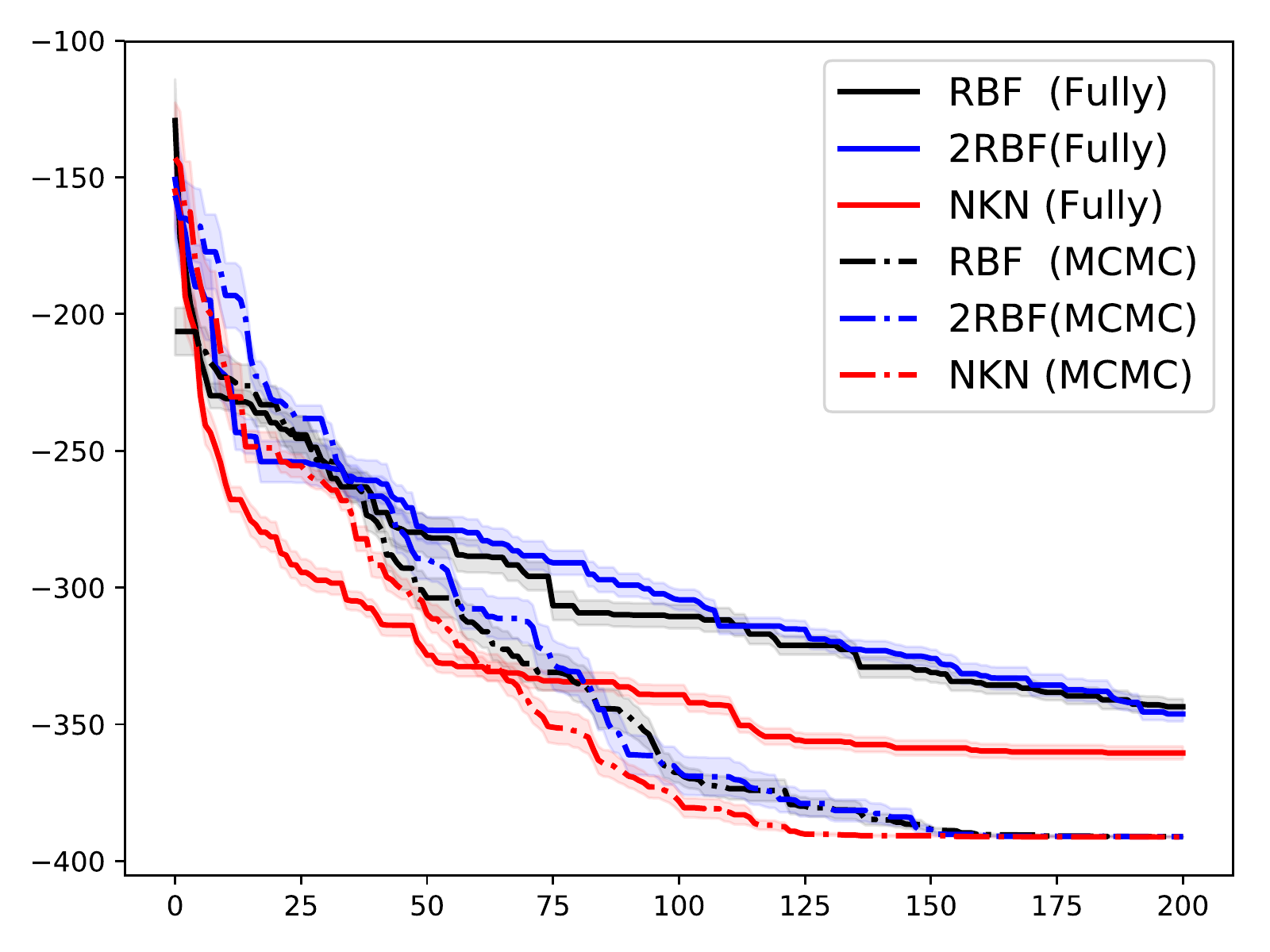} 
} 
\vspace{-1em}
\caption{Bayesian optimization on three tasks. Here 2RBF represents $\text{RBF}+\text{RBF}$. \textit{Fully} and \textit{MCMC} correspond to Fully dependent and MCMC \citep{gardner2017discovering} models, respectively. Shaded error bars are 0.2 standard deviation over 10 runs.}
\label{fig:bo-model-mcmc} 
\end{figure*}

We perform another Bayesian Optimization experiment compared to Sec~\ref{subsec:bo}. Sec~\ref{subsec:bo} models the function but optimizes function for all dimensions together, which doesn't take advantage of the additive strucutre. However, because additive kernels correspond to additive functions, we can optimize the function within each additive groups and combine them together, which is much more sample efficiency \citep{kandasamy2015high, gardner2017discovering, wang2017batched}.

In this experiment,  we adopt a small NKN network which takes $2$ RBF as primitive kernels, following layers arranged as Linear4-Product2-Linear1. We compare NKN with standard RBF and $\text{RBF}+\text{RBF}$, which we denote as 2RBF.

We compare these kernels on two models. The fully dependent model uses a single $d$-dimensional kernel for fitting, which is the standard approach in Bayesian optimization. The Model MCMC~\citep{gardner2017discovering} samples partitions using MCMC from the posterior. We plot the optimization results in Figure~\ref{fig:bo-model-mcmc}. Each curve shows the mean cumulative minimum function evaluation up to the current iteration across 10 runs, while the shaded region shows $0.2$ the standard deviation.

We found that for the fully dependent model, NKN distinctly outperforms RBF and 2RBF in terms of not only the minimum found but also convergence speed for Styblinski and Styblinski-transformed. This is because the fully dependent algorithm models correlations between all dimensions, which is beyond the capacity of RBF and 2RBF. Though NKN is expressive enough to model above two functions well, it fails to model Michalewicz function well. Therefore, we switch to use Model MCMC which can explore the additive structure of the underlying function. Stybtang task has only one local minimum, thus all three kernels perform well. In contrast, Michalewicz has many steep valleys and is much more difficult to model with simple kernels like RBF and 2RBF. Compared with Stybtang and Michalewicz, Stybtang-transformed's function introduces correlation between dimensions, which calls for a more expressive kernel structure to represent. This explains the improved convergence speed of NKN compared to the other two baselines.

%
%
\section{Implementation Details}
\label{app:details}
In this section we introduce the implementation details of NKN. For better analyzing NKN's parameter scale, we first list the parameter numbers for all the primitive kernels we use.

\begin{table}[ht]
\centering
\caption{Number of kernel parameters for $d$ dimensional inputs.}
\label{app:kernel-n-params}
\begin{tabular}{|c|c|c|c|c|c|}
\hline
RBF & LIN & COS & PER  & RQ  \\ \hline
$d+1$ & $1$   & $d+1$ & $2d+1$ & $d+1$  \\ \hline
\end{tabular}
\end{table}

\subsection{Toy experiments} \label{app:toy-details}
Across all 1-D and 2-D toy experiments, we adopt the same architecture and the same hyper-parameters initialization strategy. Concretely, we provide 8 primitive kernels ($2$ RBFs, $2$ PERs, $2$ LINs, and $2$ RQs respectively) as the primitive kernels. The NKN has a Linear8-Product4-Linear4-Product2-Linear1 structure. This NKN has overall $111$ trainable parameters. We optimize our model using Adam with an initial learning rate of $0.001$ for $20,000$ iterations.

\subsection{Regression} \label{app:reg-detials}
\label{app:reg-classification}

For the regression benchmarks, we follow the standard settings of \citep{hernandez2015probabilistic} for bayesian neural networks. We use the network with 1 hidden layer of $50$ units for all the datasets, except for \textit{protein} we use 1 hidden layer of $100$ units. We also compare NKN with standard RBF kernel and SM kernels. For the SM kernel, we use 4 mixtures which has $8d+12$ trainable parameters for $d$ dimensional inputs. The NKN has 6 primitive kernels including 2 RQ, 2 RBF, and 2 LIN. The following layers organize as Linear8-Product4-Linear4-Product2-Linear1. This architecture has $4d+85$ trainable parameters. 

Because GP suffers from $O(N^3)$ computational cost, which brings up difficulties training for large datasets. Therefore, we use Variational Free Energy (VFE) \citep{titsias2009variational} to train the GP models for datasets with more than $2000$ data points, while for the small datasets we use the vanilla GP. 

\subsection{Bayesian Optimization}
We first introduce two standard optimization benchmarks that have fully additive structure: the Styblinski-Tang function and the Michalewicz function. The $d$-dimensional Styblinski-Tang function is defined as 
\begin{equation}
    \mathrm{Stybtang}(\bx) = \frac{1}{2}\sum_{i=1}^d x_i^4 - 16x_i^2 + 5x_i
\end{equation}
which obtains its global minimum at approximately at $\bx^* \approx (-2.9, \cdots, -2.9)$ with value $-39.166d$. In practice	, we limit the exploring domain within $[-4, 4]^d$ and set $d=10$.

The $d$-dimensional Michalewicz function is defined as
\begin{equation}
    \mathrm{Michalewicz}(\bx) = - \sum_{i=1}^d \sin(x_i) \sin^{2m}(\frac{ix_i^2}{\pi})
\end{equation}
Here $m$ controls the steepness of valleys and ridges; a larger $m$ leads to a more difficult search. We set $m=10, d=10$, then the global minimum is approximately $-9.66$ over the domain $[0, \pi]^d$.

In addition, we extend the Styblinski-Tang function to a transformed Styblinski-Tang function that is not fully additive. We sample a random partition $P$ and for each part $i$ of $P$, we sample a random orthonormal matrix $Q_i$ over the dimensions of part $i$. If $Q$ is the block diagonal matrix formed by placing each $Q_i$ on the diagonal, then Stybtang($Qx$) is no longer fully additive, but instead is additive across the components of $P$. This evaluates performance of BO algorithms when the true function has corrlations between inputs dimensions.

For the Bayesian optimization process, we firsly sample 10 initial data points randomly. Then at each step, we use a GP with the proposed kernel to fit all data points available. We choose the next point from all candidate points in the exploring domain according to \textit{expected improvement}. For sampling in Model MCMC, we perform 30 steps in the first 30 iterations, while we perform 10 steps after. Among all models and tasks, we use L-BFGS optimizer for speed. In this experiment, we use shared lengthscale for all input dimensions, in which case the trainable parameters for RBF is 2. 

\subsection{Texture Extrapolation}
In texture extraplation, all observations don't exactly on a grid. Following the algorithm in, we complete the grid using extra $W$ imaginary observations, $\mathbf{y}_W \sim \mathcal{N}(\mathbf{f}_W, \epsilon^{-1}\mathbf{I}), \epsilon \rightarrow 0$. In practice, $\epsilon$ is set to $1e^{-6}$. The total observation vector $\mathbf{y} = [\mathbf{y}_M, \mathbf{y}_W]^\top$ has $N=M+W$ entries. 

We use preconditioned conjugate gradients (PCG) to compute $(K_N + D_N)^{-1}\mathbf{y}$, we use the preconditioning matrix $C = D_N^{-1/2}$ to solve $C^\top(K_N+D_N)C\mathbf{z} = C^\top\mathbf{y}$. 

To compute the log-determinant term in marginal likelihood, we cannot efficiently decompose $K_M + D_M$ since $K_M$ is not a Kronecker matrix, Considering 
\begin{equation}
	\log |K_M + D_M| = \sum_{i=1}^M \log(\lambda_i^M + \sigma^2) \approx \sum_{i=1}^M \log(\tilde{\lambda}_i^M + \sigma^2),
\end{equation}
where $\sigma$ is the noise standard deviation of the data. We approximate the eigenvalues $\lambda_i^M$ of $K_M$ using the eigenvalues of $K_N$ such that $\tilde{\lambda}_i^{M} = \frac{M}{N}\lambda_i^N$. Since $\lambda_i^n$ is the eigenvalue of a matrix with Kronecker structure, it can be computed efficiently.

\end{document}